\documentclass{article}

\usepackage{microtype}
\usepackage{graphicx}
\usepackage{subfigure}
\usepackage{booktabs} 

\usepackage{hyperref}
\usepackage{colortbl}

\usepackage[accepted]{icml2024}

\usepackage{amsmath}
\usepackage{amssymb}
\usepackage{mathtools}
\usepackage{amsthm}
\usepackage{multirow}
\usepackage{makecell}
\usepackage[capitalize,noabbrev]{cleveref}

\theoremstyle{plain}
\newtheorem{theorem}{Theorem}[section]
\newtheorem{proposition}[theorem]{Proposition}

\theoremstyle{definition}
\newtheorem{definition}[theorem]{Definition}

\theoremstyle{remark}

\usepackage[textsize=tiny]{todonotes}

\icmltitlerunning{Sample-specific Masks for Visual Reprogramming-based Prompting}

\begin{document}

\twocolumn[
\icmltitle{Sample-specific Masks for Visual Reprogramming-based Prompting}




\begin{icmlauthorlist}
\icmlauthor{Chengyi Cai}{Unimelb}
\icmlauthor{Zesheng Ye}{Unimelb}
\icmlauthor{Lei Feng}{NTU}
\icmlauthor{Jianzhong Qi}{Unimelb}
\icmlauthor{Feng Liu}{Unimelb}
\end{icmlauthorlist}

\icmlaffiliation{Unimelb}{School of Computing and Information Systems, The University of Melbourne}
\icmlaffiliation{NTU}{Information Systems Technology and Design Pillar, Singapore University of Technology and Design}

\icmlcorrespondingauthor{Feng Liu}{fengliu.ml@gmail.com}

\icmlkeywords{Machine Learning, ICML}

\vskip 0.3in
]



\printAffiliationsAndNotice{}  

\begin{abstract}
\emph{Visual reprogramming} (VR) is a prompting technique that aims to re-purpose a pre-trained model (e.g., a classifier on ImageNet) to target tasks (e.g., medical data prediction) by learning a \emph{small-scale pattern} added into input images instead of tuning considerable parameters within the model. 
The location of the pattern within input samples is usually determined by a pre-defined mask \emph{shared across all samples}. 
In this paper, we show that the shared mask potentially limits VR's generalization and increases its approximation error due to the lack of sample-level adaptation.
Motivated by this finding, we design a new framework for VR called \emph{sample-specific multi-channel masks} (SMM). 
Specifically, SMM employs a lightweight ConvNet and patch-wise interpolation to generate sample-specific three-channel masks instead of a shared and pre-defined mask.
Since we generate different masks for individual samples, SMM is theoretically shown to reduce approximation error for the target tasks compared with existing state-of-the-art VR methods. We also empirically demonstrate its performance gain on both ResNet and ViT.
The success of SMM further highlights the broader applicability of VR in leveraging the latent knowledge of pre-trained models for various target tasks. Our code is available at \url{https://github.com/tmlr-group/SMM}.
\end{abstract}

\section{Introduction}
Recent studies have shown that, by taking advantage of and re-purposing well-trained/pre-trained models, one can address new tasks (i.e., \emph{target tasks}) without training a task-specific model from scratch  \cite{basu2023efficient,kossen2023three,mondal2023equivariant}. 
In visual tasks, due to the expensive training costs even just to finetune pre-trained models, \emph{visual reprogramming} (VR) \cite{neekhara2022cross,wang2022watermarking,chen2023understanding,tsao2024autovp}, or adversarial reprogramming \cite{elsayed2018adversarial,tsai2020transfer}, has been proposed to reuse pre-trained models on target tasks. 
Concretely, VR is a prompting method that fixes a pre-trained model and only alters the input space by adding some learnable patterns (usually some noise) to target images. 
The location of the patterns to be learned is usually determined by a pre-defined binary mask that is \emph{shared across all samples} \cite{elsayed2018adversarial,yang2021voice2series, tsai2020transfer,bahng2022exploring}.
The key benefit of VR methods is that learning the pattern whose size is around the image size requires much less computing resource than finetuning considerable parameters within the model, posing VR as a promising research area in using pre-trained models \cite{chen2023understanding,tsao2024autovp}.

\begin{figure}
    \centering
    \includegraphics[width=\linewidth]{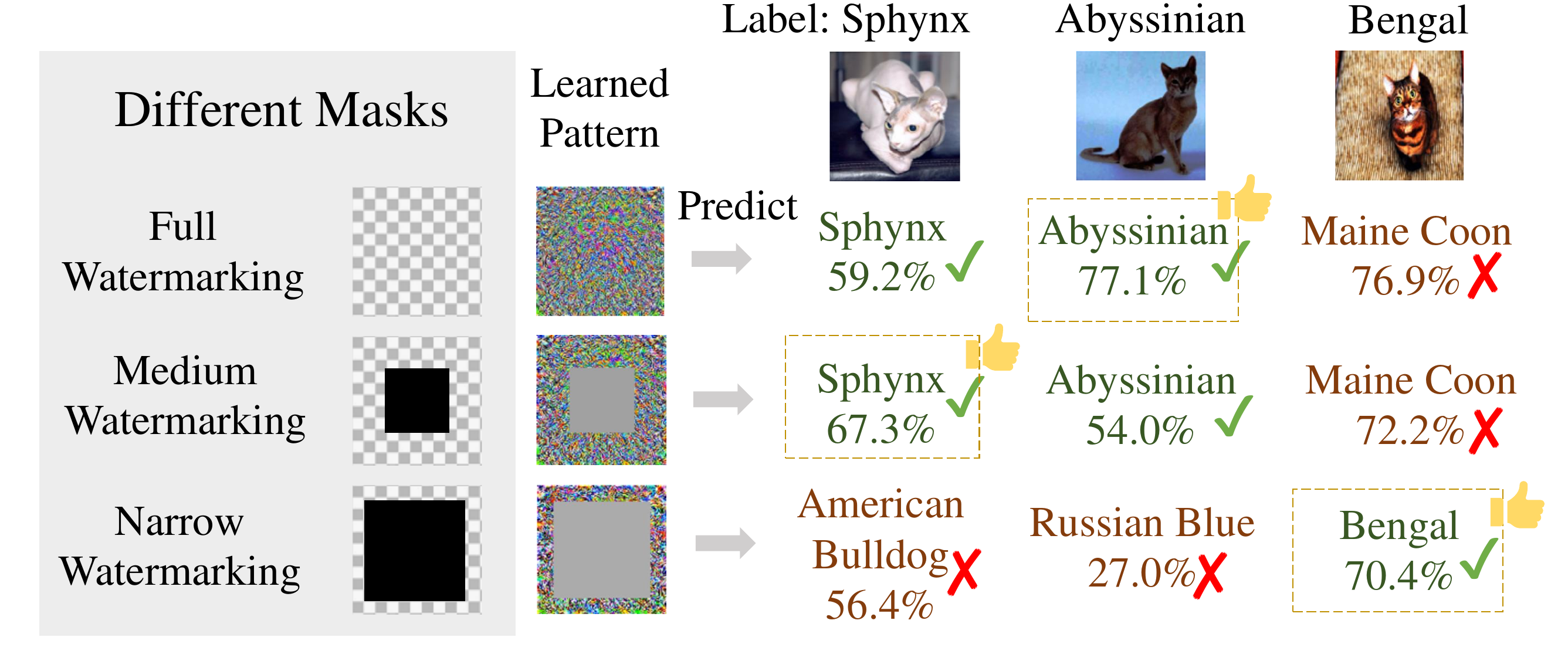}
    \caption{Drawback of shared masks over individual images. 
    We demonstrate the use of watermarking~\cite{wang2022watermarking}, a representative VR method, to re-purpose an ImageNet-pretrained classifier for the OxfordPets dataset, with different shared masks (full, medium, and narrow) in VR. An evaluation of classification confidence across three cat images — Sphynx, Abyssinian, and Bengal — indicates a sample-specific mask preference: Sphynx with medium, Abyssinian with full, and Bengal with narrow. It shows that different masks are needed for individual images.
    }
    \label{fig:motivaton2}
\end{figure}

\begin{figure*}
    \centering
    \includegraphics[width=\linewidth]{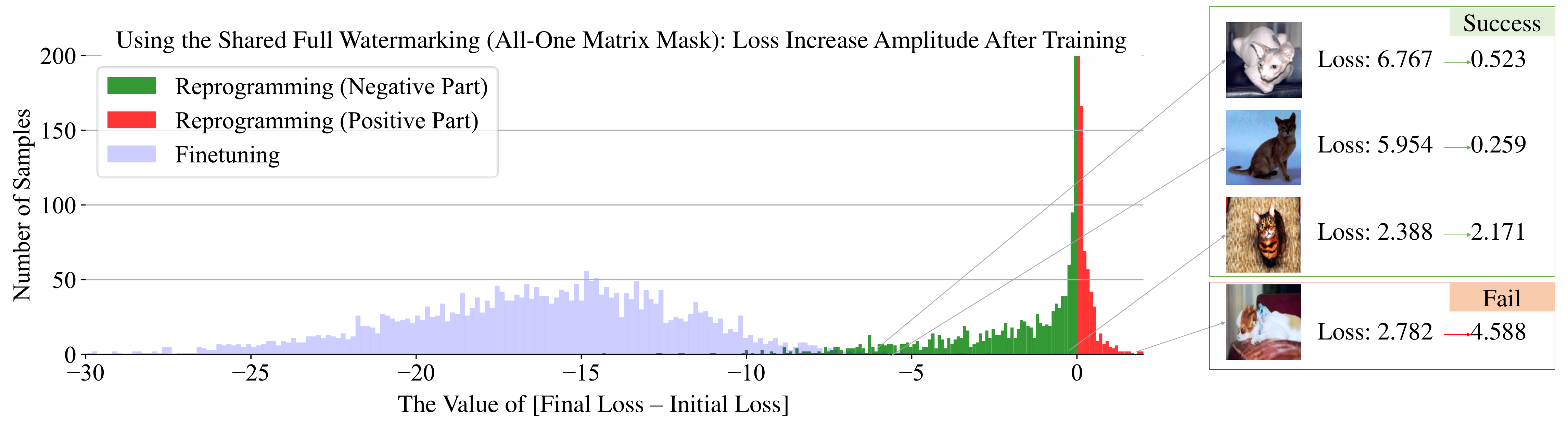}
    \caption{Drawback of shared masks in the statistical view. Optimal learning methods like finetuning usually result in loss decreases for all samples (see the blue part). But when applying the same mask in reprogramming, part of the loss changes are observed to be positive (see the red part) according to the distribution of [final loss - initial loss], which means the training loss for some samples even rises.}
    \label{fig:motivation1}
\end{figure*}

In this paper, we show that the shared mask often leads to poor generalization capability of VR, as demonstrated in Figures~\ref{fig:motivaton2} and \ref{fig:motivation1}. In both figures, we use a representative VR method, watermarking \cite{bahng2022exploring}, to re-purpose an ImageNet-pretrained classifier to classify images in the OxfordPets datasets \cite{parkhi2012cats}.
In Figure \ref{fig:motivaton2}, we first find that the optimal masks vary among individual images. We apply three kinds of masks (full, medium, and narrow) in watermarking. By observing the classification confidence on three cat images: Sphynx, Abyssinian, and Bengal, we see that the medium mask is the best for Sphynx, the full mask for Abyssinian, and the narrow mask for Bengal. This suggests that different masks are needed for individual images.
In Figure~\ref{fig:motivation1}, we then find that watermarking with a single shared mask may cause the training loss of many individual samples to rise (see the red part in Figure~\ref{fig:motivation1}). This phenomenon reveals that VR methods' learning capacity is much less than finetuning all parameters of the pre-trained model (see the blue part in Figure~\ref{fig:motivation1}). 

The examples above show a significant disadvantage of using a single shared mask for VR. This motivates our new VR framework called \emph{sample-specific multi-channel masks} (SMM). SMM replaces the \emph{fixed} binary mask applied in existing works with \emph{generative} three-channel masks that can vary across different samples (shown in Figure~\ref{fig:framework}). 

SMM has two modules: a mask generator module and a patch-wise interpolation module. The mask generator is a lightweight \emph{convolutional neural network} (CNN) that takes resized individual target-domain images (i.e., samples) as the input and outputs different masks for each sample. The last layer of the generator is designed to generate a three-channel mask, which allows better performance for both rich-color images (i.e., CIFAR10/100 \cite{krizhevsky2009learning}) and monotonous-color images (i.e., SVHN \cite{yuval2011reading}). Since the generated masks should be the same size as the pattern to be learned, when the size of masks is inconsistent with that of the pattern, the patch-wise interpolation module will be utilized to re-scale the generated masks to fit the pattern, facilitating the training process of the mask generator (detailed in Section~\ref{method}).

To understand why SMM is effective, we theoretically analyze the approximation error of different hypothesis sets for VR. Three hypothesis sets are considered: shared pattern with a pre-defined binary mask, sample-specific patterns without masks, and our SMM. We show that SMM has a smaller approximation error (Proposition~\ref{proposition1}), which confirms the effectiveness of SMM. 

To further substantiate the efficacy of SMM, we conduct empirical evaluations spanning 11 widely used datasets, incorporating ablation studies that discern the impact of individual SMM components.
This is complemented by analysis and interpretations of the generated masks, alongside a comparative visualization of feature spaces.
Notably, we demonstrate the effectiveness of SMM with both pre-trained ResNet~\cite{he2016deep} and ViT~\cite{dosovitskiy2020image} (Table~\ref{resnetresults} and \ref{ViTresults}), validating that SMM is compatible with commonly used classifier architectures.

Both the theoretical analysis and promising experimental results provide solid evidence that, when powered by SMM, VR can efficiently leverage knowledge within a well-trained/pre-trained model for various target tasks, shedding new light on the explanatory analysis of VR and opening avenues for future research.

\section{Preliminaries and Related Works}
\subsection{Problem Setting of Model Reprogramming}
Model reprogramming \citep{chen2022survey} offers an efficient transfer learning paradigm for adapting pre-trained models to resource-constrained target tasks.
This paradigm re-purposes existing knowledge by strategically transforming inputs and outputs, bypassing extensive model parameter finetuning.
In what follows, we will present a formal problem setting for model reprogramming.

Let $\mathcal{D}_{\rm T}$ represent the data distribution of a target task defined over $\mathcal{X}^{\rm T}\times \mathcal{Y}^{\rm T}$, where $\mathcal{X}^{\rm T}\subseteq  \mathbb{R}^{d_{\rm T}}$ is the data space and $\mathcal{Y}^{\rm T}=\{1,\dots,k_{\rm T}\}$ is the label space, and let $\{(x^{\rm T}_{i},y^{\rm T}_i)\}_{i=1}^n$ be the observations of $\mathcal{D}_{\rm T}$ (i.e., the training set in the target task). Meanwhile, we have a pre-trained model $f_{\rm P}:\mathcal{X}^{\rm P}\rightarrow \mathcal{Y}^{\rm P}$, where $\mathcal{X}^{\rm P} \subseteq \mathbb{R}^{d_{\rm P}}$ and $\mathcal{Y}^{\rm P}$ (s.t. $|\mathcal{Y}^{\rm T}| \leq |\mathcal{Y}^{\rm P}|$, with the label space of the pre-trained task being larger than that of the target task) represent the data and label spaces used for training $f_{\rm P}$. 
Then, in model reprogramming, the training objective can be formulated as
\begin{align}\label{eq:vp_obj}
    \mathop{\min}\limits_{\theta\in\Theta, \omega\in\Omega}\frac{1}{n}\sum_{i=1}^n\ell(f_{\rm out}(f_{\rm P}(f_{\rm in}(x^{\rm T}_i|\theta))|\mathcal{Y}^{\rm P}_{\rm sub},\omega), y^{\rm T}_i),
\end{align}
where $f_{\rm in}(.|\theta): \mathcal{X^{\rm T}} \mapsto  {\mathcal{X}}^{\rm P}, f_{\rm out}(.|\mathcal{Y}^{\rm P}_{\rm sub},\omega): \mathcal{Y}^{\rm P}_{\rm sub} \mapsto \mathcal{Y^{\rm T}}$ are the input transformation and output label mapping function with parameters $\theta\in\Theta$ and $\omega\in\Omega$, $\mathcal{Y}^{\rm P}_{\rm sub}\subseteq\mathcal{Y}^{\rm P}$ can be determined by different methods \cite{elsayed2018adversarial,tsai2020transfer,chen2023understanding}, and $\ell:\mathcal{Y}^{\rm T}\times\mathcal{Y}^{\rm T}\mapsto \mathbb{R}^+\cup\{0\}$ is a loss function.
Reprogramming techniques have been widely applied in visual \cite{elsayed2018adversarial,tsai2020transfer}, text \cite{neekhara2018adversarial,hambardzumyan2021warp}, speech \cite{yang2021voice2series,yang2023english,yen2023neural}, music \cite{hung2023low}, and cross-modal tasks \cite{neekhara2022cross} in the past few years. 

In the context of visual tasks, reprogramming has demonstrated potential in bio-medical measurement \citep{tsai2020transfer}, machine learning fairness \citep{zhang2022fairness},
as well as out-of-distribution detection through 
watermarking \citep{wang2022watermarking}.
Moving beyond application prospects, we next discuss the technical details of the specific input and output mapping functions ($f_{\rm in}$ and $f_{\rm out}$).

\subsection{Prompting and Input Visual Reprogramming}
General prompting methods in visual tasks, predominantly applied to the ViT architecture \cite{dosovitskiy2020image}, introduce extra parameters to a pre-trained model for enhanced training efficiency.
Prompts are flexible in their placement. 
For example, \emph{visual prompt tuning} \cite{jia2022visual} positions prompts alongside image embedding before the encoder layers, while \emph{effective and efficient visual prompt tuning} \cite{han20232vpt} extends this by incorporating parameters within self-attention layers as well. \emph{Transformer with hierarchical prompting} \cite{wang2023transhp} also learns prompt tokens to represent the coarse image classes.

Meanwhile, prompting goes beyond vision foundation models to vision-language frameworks such as CLIP \cite{radford2021learning}. 
Methods like CoOP \cite{zhou2022learning} and CoCoOP \cite{zhou2022conditional} replace textual prompts with learnable vectors for enhanced adaptability to the target task, conditioned on input images.
MaPLe \cite{khattak2023maple} further bridges vision and language by learning layer-specific mapping functions. 
These methods vary from each other in terms of both prompt placements and functions.

In contrast, VR provides a model-agnostic prompting technique, by adding trainable noise to the input image patterns before the forward propagation, without altering their visual essence. 
Originally proposed by~\citet{elsayed2018adversarial}, VR has been evolving to include padding-based methods \cite{tsai2020transfer,chen2023understanding} and watermarking that facilitate downstream target tasks \cite{bahng2022exploring}.
AutoVP \cite{tsao2024autovp} stands out with its scalable pre-padding images. 
A critical limitation in existing VR research is the use of \emph{shared} noise patterns across \emph{all target samples}, neglecting sample-level characteristics and compromising generalization.
We propose SMM to manage this gap.

\subsection{Output Mapping of Reprogramming}
Learning-based output mapping, i.e., model $f_{\rm out}$, as proposed by \citet{chen2023understanding}, can be simplified as a one-to-one mapping from a subset of $\mathcal{Y^{\rm P}}$ to $\mathcal{Y^{\rm T}}$. Therefore, no additional parameters are required. 
One implementation of this mapping is \emph{random label mapping} (Rlm), where $f_{\rm out}$ is a randomly assigned injective function \cite{elsayed2018adversarial,chen2023understanding}, formulated as
\begin{align}\label{eq:rlm}
    f_{\rm out}^{\rm Rlm}(y|\mathcal{Y}^{\rm P}_{\rm sub})=\texttt{rand}(\{0, 1, ...,k^{\rm T}\}),
\end{align}
where $\texttt{rand}(\{0, 1, ...,k^{\rm T}\})$ means randomly selecting one element from the set $\{0, 1, ...,k^{\rm T}\}$, and $\mathcal{Y}^{\rm P}_{\rm sub}$ is of the same size with $\mathcal{Y}^{\rm T}$ (i.e., $k^{\rm T}$), randomly chosen from $\mathcal{Y}^{\rm P}$ prior to the minimization of Eq.~\eqref{eq:vp_obj}. Note that, since $f_{\rm out}^{\rm Rlm}$ is injective, it ensures $f_{\rm out}^{\rm Rlm}(y_1|\mathcal{Y}^{\rm P}_{\rm sub})\neq f_{\rm out}^{\rm Rlm}(y_2|\mathcal{Y}^{\rm P}_{\rm sub})$ for two distinct elements $y_1\neq y_2$. 

Other representative output-mapping methods determine $\mathcal{Y}^{\rm P}_{\rm sub}$ and $f_{\rm out}$ for different target tasks. For example, one is based on the frequency of label assignment in the pre-trained model and the target data \cite{tsai2020transfer}, called \emph{frequent label mapping} (Flm). \citet{chen2023understanding} propose \emph{iterative label mapping} (Ilm) that updates $f_{\rm out}$ in each training iteration, reflecting changes in label mapping throughout the learning of $f_{\rm in}$. Detailed procedures and the pseudo-code of ${f}_{\rm out}^{\rm Flm}$ and ${f}_{\rm out}^{\rm Ilm}$ are in Appendix~\ref{app:out}.

\begin{figure*}
    \centering
    \includegraphics[width=\linewidth]{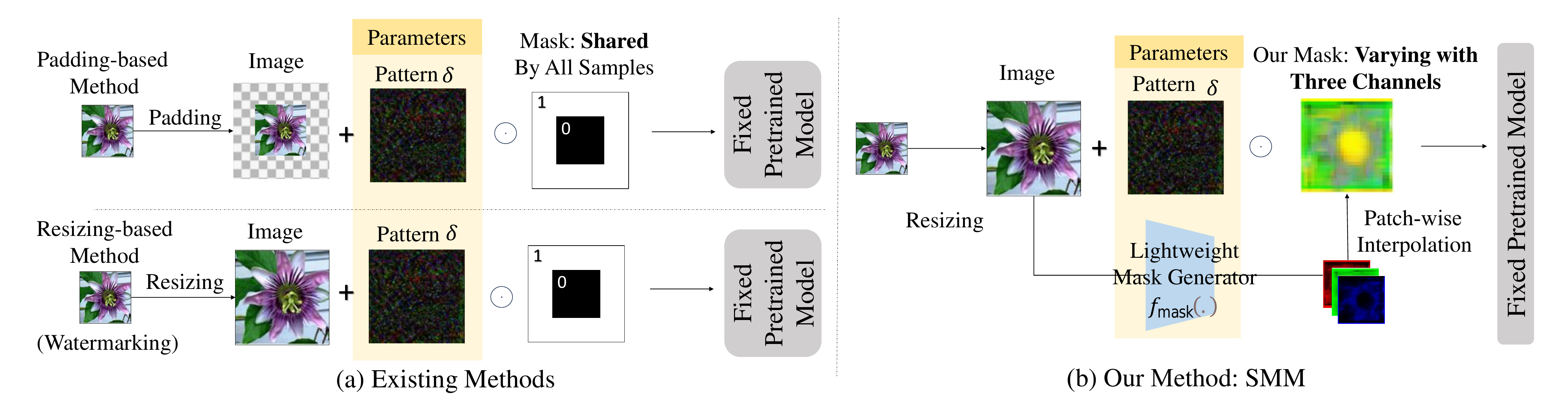}
    \caption{Comparison between (a) existing methods and (b) our method. Previous padding-based reprogramming adds zeros around the target image, while resizing-based reprogramming adjusts image dimensions to fit the required input size. Both methods use a pre-determined \emph{shared} mask to indicate the valid location of pattern $\delta$. Our method, on the other hand, takes a more dynamic and tailored approach. We resize each target image and apply a different three-channel mask accordingly, driven by a lightweight $f_{\rm mask}$ with an interpolation up-scaling module, allowing for more variability in individual samples.}
    \label{fig:framework}
\end{figure*}

\section{Sample-specific Multi-channel Masks} \label{method}
We focus on $f_{\rm in}$, while treating $f_{\rm out}$ as a non-parametric mapping, in line with \citet{chen2023understanding}. 
We thus limit our discussion of trainable parameters to $\theta\in \Theta$ in Eq.~\eqref{eq:vp_obj}.
A flowchart in Appendix \ref{prob} provides an overview of the problem structure of Input VR.
\subsection{Framework of SMM}

To allow both shared parameters over the whole dataset and variability among individual samples, it is intuitive to consider the following VR hypothesis: 
\begin{align}
f_{\rm in}(x_i|\phi,\delta)=r(x_i)+\delta \odot f_{\rm mask}(r(x_i)|\phi), 
\end{align}
where $r:\mathcal{X}^{\rm T}\rightarrow \mathbb{R}^{d_{\rm P}}$ is the resizing function, typically implemented as bilinear interpolation upsampling \cite{bilinear} that scales image dimension from $d_{\rm{T}}$ to $d_{\rm{P}}$, and $r(x_i)\in \mathbb{R}^{d_{\rm P}}$ is the resized image corresponding to $x_i$.
The mask generation function $f_{\rm mask}:{\mathbb{R}}^{d_{\rm P}}\rightarrow \mathbb{R}^{d_{\rm P}}$, parameterized by $\phi\in\Phi$, produces a mask indicating the noise placements for each image.
We denote a trainable noise pattern added to the image by $\delta \in \mathbb{R}^{d_{\rm P}}$.
The rationale for applying this hypothesis is elaborated in Proposition \ref{proposition1} and validated in ablation studies (cf.  Table~\ref{ablation}). 
This casts the training objective of our SMM framework ($\theta = \left\{ \phi, \delta \right\}$) to find the optimal $\phi^{*}$ and $\delta^{*}$ such that
\begin{equation}\label{eq:final_obj}
\begin{split}
    \mathop{\arg \min}\limits_{\phi\in\Phi, \delta\in\mathbb{R}^{d_{\rm P}}}\mathbb{E}_{(x_i, y_i) \sim \mathcal{D_{\rm T}}}[\ell( & f_{\rm out}(f_{\rm P}(r(x_i)+\\
    & \delta \odot f_{\rm mask}(r(x_i)|\phi))),y_i)].
\end{split}
\end{equation}
Note that $\delta$ is \emph{shared} by all images in the dataset following \citet{bahng2022exploring} and \citet{chen2023understanding}, while $f_{\rm mask}$ \emph{uniquely} generates sample-specific multi-channel masks for each individual image, enabling sample-specific adaptation.

Figure~\ref{fig:framework} illustrates the workflow of SMM, as well as previous padding-based and resizing-based (i.e., watermarking) VR methods.
Compared with previous works, SMM features $f_{\rm mask}(\cdot|\phi)$, integrating a \textit{mask generator module} and a \textit{patch-wise interpolation module}. 
Concretely, SMM starts by resizing target images, followed by their processing through the mask generator to create corresponding three-channel masks.
For generated masks smaller than the pattern size, the patch-wise interpolation module performs upsampling, which omits the derivation step in back-propagation and facilitates training.
Afterward, the learnable pattern $\delta$ is multiplied with the mask on a pixel-wise basis and added to the image. The resulting image is fed into the \emph{fixed} pre-trained classifier. 
We discuss further details on the mask generator (Section ~\ref{sec:generator}), the patch-wise interpolation module (Section~\ref{sec:patch_priority}), and the overall learning strategy presented in Eq.~\eqref{eq:final_obj} (Section~\ref{sec: learning_strategy}).

\subsection{Lightweight Mask Generator Module}\label{sec:generator}
The mask generator $f_{\rm mask}$ is supposed to output a mask that has the same size as the input image while prioritizing different locations for $\delta$ to allow more variability.
We employ a CNN as the mask generator. 
This choice stems from the proficiency of CNNs in mirroring localized visual perception \cite{he2016deep} with fewer parameters than most deep learning structures, e.g., multilayer perceptrons.

The input of CNN is a resized image $r(x_i)$. Applying our bespoke CNN architecture shown in Appendix \ref{architecture}, the output will be a three-channel mask with dimensions $\left \lfloor  \frac{H}{2^l}\right \rfloor  \times \left \lfloor \frac{W}{2^l}\right \rfloor $, where $H$ and $W$ denote image height and width, respectively, and $l$ denotes the number of pooling layers. 
The analysis of input/output sizes and parameter quantity statistics are in Appendix \ref{architecture}. 

\begin{algorithm}[tb]
   \caption{Visual Reprogramming with SMM}
   \label{alg:method} 
\begin{algorithmic}[1]
   \STATE {\bfseries Input:} Pre-trained model $f_{\rm P}$, loss $\ell$, label-mapping function $f^{(j)}_{\rm out}$ for iteration $j$, target domain training data $\{(x_i, y_i)\}^{n}_{i=1}$, maximum number of iterations $E$, learning rate $\alpha_1$ for $\delta$ and $\alpha_2$ for $\phi$
   \STATE {\bfseries Output:} Optimal $\delta^*, \phi^*$
   \STATE Initialize $\phi$ randomly; set $\delta \leftarrow \{0\}^{d_\mathcal{\rm P}}$
   \FOR{$j=1$ {\bfseries to} $E$}
   \STATE \# Step1: Compute individual marks using the mask generator\\
   \# Step2: Resize masks using the patch-wise interpolation module \\
     $f_{\rm in}(x_i;\delta, \phi)\leftarrow r(x_i) + \delta \odot f_{\rm mask}(r(x_i)|\phi)$, $\forall i=1, 2, ..., n$
   \STATE \# Compute the classification loss \\
   $L(\delta, \phi)\leftarrow\frac{1}{n}\sum^{n}_{i=1}\ell(f^{(j)}_{\rm out}{(f_{\rm P}(f_{\rm in}(x_i;\delta, \phi)))}, y_i)$
    \STATE $\delta \leftarrow \delta - \alpha_1 \nabla_\delta L(\delta, \phi)$
    \STATE $\phi \leftarrow \phi - \alpha_2 \nabla_\phi L(\delta, \phi)$
    \ENDFOR
\end{algorithmic}
\end{algorithm}
\subsection{Patch-wise Interpolation Module}\label{sec:patch_priority}
The patch-wise interpolation module upscales CNN-generated masks from $\left \lfloor  \frac{H}{2^l}\right \rfloor  \times \left \lfloor \frac{W}{2^l}\right \rfloor $ back to the original size $H \times W$ per channel (it is omitted when $l = 0$).
Considering the inherent consistency in adjacent image areas and the benefits of concise operations for gradient calculations, we employ a grid of $\left \lfloor  \frac{H}{2^l}\right \rfloor  \times \left \lfloor \frac{W}{2^l}\right \rfloor $ patches in the upsampling process, each sized $2^l \times 2^l$, ensuring the same values within each patch, with non-divisible cases mirroring the closest patches.
Therefore, after obtaining the output of CNN, we enlarge each pixel to $2^l \times 2^l$ pixels by padding the same value of a pixel to its surrounding areas within the patch.

Unlike traditional interpolation methods which may introduce complicated derivation computations, our module simplifies the training by directly assigning values. 
The advantage of patch-wise interpolation over traditional interpolation methods will be discussed in Appendix~\ref{app:interpolation}.
The effect of patch size $2^l$ will be discussed in Section \ref{exp}.

\subsection{Learning Strategy} \label{sec: learning_strategy}
The learning process for the shared noise pattern $\delta$ and the mask generator $f_{\rm mask}$ is shown in Algorithm \ref{alg:method}. The parameters $\delta$ and $\phi$ are iteratively updated in each epoch. To mitigate the impact of initialization, $\delta$ is set to be a zero matrix before training, noted as $\{0\}^{d_\mathcal{\rm P}}$.

\begin{table*}[!ht] 
\caption{Performance Comparison of Different Input Reprogramming Methods on Pre-trained ResNet (Mean \% ± Std \%, the average results across all datasets are highlighted in grey)}
\label{resnetresults}

\begin{center}
\begin{small}
\begin{sc}
\resizebox{\textwidth}{!}{
\begin{tabular}{c|ccccc|ccccc} 
\toprule
Pre-trained & \multicolumn{5}{c|}{ResNet-18 (ImageNet-1k)}                            & \multicolumn{5}{c}{ResNet-50 (ImageNet-1k)}               \\
\midrule
 Methods & Pad          & Narrow   & Medium   & Full     & Ours              & Pad  & Narrow   & Medium   & Full     & Ours              \\
\midrule
CIFAR10                        & 65.5  \scriptsize±0.1          & 68.6 \scriptsize±2.8 & 68.8 \scriptsize±1.1 & 68.9 \scriptsize±0.4 & \textbf{72.8 }\scriptsize±0.7 & 76.6\scriptsize±0.3 & 77.4\scriptsize±0.5 & 77.8\scriptsize+0.2 & 79.3\scriptsize±0.3 & \textbf{81.4}\scriptsize±0.6 \\
CIFAR100                       & 24.8\scriptsize±0.1          & 36.9\scriptsize±0.6 & 34.9\scriptsize±0.2 & 33.8\scriptsize±0.2 & \textbf{39.4}\scriptsize±0.6 & 38.9\scriptsize±0.3 & 42.5\scriptsize±0.2 & 43.8\scriptsize±0.2 & 47.2\scriptsize±0.1 & \textbf{49.0}\scriptsize±0.2 \\
SVHN                           & 75.2\scriptsize±0.2          & 58.5\scriptsize±1.1 & 71.1\scriptsize±1.0 & 78.3\scriptsize±0.3 & \textbf{84.4}\scriptsize±2.0 & 75.8\scriptsize±0.4 & 59.1\scriptsize±1.3 & 71.5\scriptsize±0.8 & 79.5\scriptsize±0.5 & \textbf{82.6}\scriptsize±2.0 \\
GTSRB                          & 52.0\scriptsize±1.2          & 46.1\scriptsize±1.5 & 56.4\scriptsize±1.0 & 76.8\scriptsize±0.9 & \textbf{80.4}\scriptsize±1.2 & 52.5\scriptsize±1.4 & 38.9\scriptsize±1.3 & 52.6\scriptsize±1.3 & 76.5\scriptsize±1.3 & \textbf{78.2}\scriptsize±1.1 \\
Flowers102                     & 27.9\scriptsize±0.7          & 22.1\scriptsize±0.1 & 22.6\scriptsize±0.5 & 23.2\scriptsize±0.5 & \textbf{38.7}\scriptsize±0.7 & 24.6\scriptsize±0.6 & 19.9\scriptsize±0.6 & 20.9\scriptsize±0.6 & 22.6\scriptsize±0.1 & \textbf{35.9}\scriptsize±0.5 \\
DTD                            & \textbf{35.3}\scriptsize±0.9 & 33.1\scriptsize±1.3 & 31.7\scriptsize±0.5 & 29.0\scriptsize±0.7 & 33.6\scriptsize±0.4          & 40.5\scriptsize±0.5 & 37.8\scriptsize±0.7 & 38.4\scriptsize±0.2 & 34.7\scriptsize±1.3 & \textbf{41.1}\scriptsize±1.1 \\
UCF101                         & 23.9\scriptsize±0.5          & 27.2\scriptsize±0.9 & 26.1\scriptsize±0.3 & 24.4\scriptsize±0.9 & \textbf{28.7}\scriptsize±0.8 & 34.6\scriptsize±0.2 & 38.4\scriptsize±0.2 & 37.2\scriptsize±0.2 & 35.2\scriptsize±0.2 & \textbf{38.9}\scriptsize±0.5 \\
Food101                        & 14.8\scriptsize±0.2          & 14.0\scriptsize±0.1 & 14.4\scriptsize±0.3 & 13.2\scriptsize±0.1 & \textbf{17.5}\scriptsize±0.1 & 17.0\scriptsize±0.3 & 18.3\scriptsize±0.2 & 18.3\scriptsize±0.2 & 16.7\scriptsize±0.2 & \textbf{19.8}\scriptsize±0.0 \\
SUN397                         & 13.0\scriptsize±0.2          & 15.3\scriptsize±0.1 & 14.2\scriptsize±0.1 & 13.4\scriptsize±0.2 & \textbf{16.0}\scriptsize±0.3 & 20.3\scriptsize±0.2 & 22.0\scriptsize±0.1 & 21.5\scriptsize±0.1 & 21.1\scriptsize±0.1 & \textbf{22.9}\scriptsize±0.0 \\
EuroSAT                        & 85.2\scriptsize±0.6          & 82.8\scriptsize±0.4 & 83.8\scriptsize±0.5 & 84.3\scriptsize±0.5 & \textbf{92.2}\scriptsize±0.2 & 83.6\scriptsize±0.7 & 83.7\scriptsize±0.4 & 85.8\scriptsize±0.1 & 86.9\scriptsize±0.3 & \textbf{92.0}\scriptsize ±0.6\\
OxfordPets                     & 65.4\scriptsize±0.7          & 73.7\scriptsize±0.2 & 71.4\scriptsize±0.2 & 70.0\scriptsize±0.6 & \textbf{74.1}\scriptsize±0.4 & 76.2\scriptsize±0.6 & 76.4\scriptsize±0.3 & 75.6\scriptsize±0.3 & 73.4\scriptsize±0.3 & \textbf{78.1}\scriptsize±0.2 \\
\rowcolor{gray!30}
Average & 43.91 & 43.48 & 45.04 & 46.85 & \textbf{52.53} & 49.15 & 46.76 & 49.39 & 52.10 & \textbf{56.35}
\\
\bottomrule
\end{tabular}}
\end{sc}
\end{small}
\end{center}
\end{table*}

\section{Understanding Masks in Visual Reprogramming for Classification}

In this section, we will demonstrate that SMM enables stronger model learning capacity than previous representative VR methods, via showing reduced approximation error in the \emph{probably approximately correct} (PAC) learning framework \cite{kearns1994introduction}. We first present the definition of the approximation error in PAC learning. 

\begin{definition}[Approximation Error] \label{definition1}
Consider an input space $\mathcal{X}$, a discrete label space $\mathcal{Y}$, a random variable $(X,Y)$ whose distribution $\mathcal{D}$ is defined on $\mathcal{X}\times\mathcal{Y}$ with a joint probability density function $p(x,y)$, and a hypothesis space $\mathcal{F}=\{f:\mathcal{X}\rightarrow\mathcal{Y}\}$. 
The approximation error of $\mathcal{F}$ on $\mathcal{D}$ is 
\begin{align}
    {\rm Err}^{\rm apx}_{\mathcal{D}}(\mathcal{F}) = \inf_{f\in\mathcal{F}}\mathbb{E}_{(X,Y)\sim\mathcal{D}}\ell(f(X),Y) - R^*_{\mathcal{D}},
\end{align}
where $\ell:\mathcal{Y}\times\mathcal{Y}\rightarrow\mathbb{R}^+\cup\{0\}$ is a loss function, and $R^*_{\mathcal{D}}$ is the Bayes risk \cite{Snapp1995est} on $\mathcal{D}$ defined by
\begin{align}
    R^*_{\mathcal{D}} = \int_\mathcal{X} \Big[1-\sup_{y\in\mathcal{Y}}{\rm Pr}(y|x)\Big]p_X(x)dx. 
\end{align}
Here, ${\rm Pr}(y|x)$ is the posterior probability of class $y$ conditioned on observing $x$, and $p_X(x)=\sum_{y\in\mathcal{Y}}p(x,y)$ is the marginal distribution of $X$.
\end{definition}
The approximation error of a hypothesis space $\mathcal{F}$ measures the closeness of the minimum achievable error by $\mathcal{F}$ to the theoretical minimum error on distribution $\mathcal{D}$. In general, increasing the complexity of $\mathcal{F}$ tends to reduce the approximation error. In the following theorem, we show a connection between two approximation errors when hypothesis spaces exhibit a subset relation.
\begin{theorem}\label{theorem1}
    Given an input space $\mathcal{X}$, a discrete label space $\mathcal{Y}$, and a distribution $\mathcal{D}$ over $\mathcal{X}\times\mathcal{Y}$, if there are two hypothesis spaces $\mathcal{F}_1\subseteq\{f:\mathcal{X}\rightarrow\mathcal{Y}\}$ and $\mathcal{F}_2\subseteq\{f:\mathcal{X}\rightarrow\mathcal{Y}\}$ satisfying that $\mathcal{F}_1\subseteq\mathcal{F}_2$, then we have ${\rm Err}^{\rm apx}_{\mathcal{D}}(\mathcal{F}_1)\geq {\rm Err}^{\rm apx}_{\mathcal{D}}(\mathcal{F}_2)$.
\end{theorem}
Theorem \ref{theorem1} (proof in Appendix \ref{app:approx}) shows that understanding the subset relation between two hypothesis spaces is key to deriving their connections in their approximation errors.
Next, we will define two hypothesis spaces: one induced by a shared mask and the other induced by SMM. 

\textbf{Hypothesis Space Induced by A Shared Mask}. VR methods with a shared mask \cite{chen2022survey,bahng2022exploring} assume that, for each sample $x_i$, the mask is a constant matrix $M\in \{0, 1\}^{d_\mathcal{\rm P}}$. Thus, given a fixed pre-trained model $f_{\rm P}$ and a fixed output mapping function $f_{\rm out}$ (for simplicity, we use $f_{\rm P}^\prime$ to represent $f_{\rm out}\circ f_{\rm P}$ in this section), the hypothesis space induced by a shared mask is
\begin{align*}
    \mathcal{F}^{\rm shr}({f_{\rm P}^\prime})=\{f|f(x) = f_{\rm P}^\prime(r(x)+M\odot\delta),\forall x\in\mathcal{X}\},
\end{align*}
where $\delta\in\mathbb{R}^{d_{\rm P}}$. In padding-based reprogramming methods, $M$ is a fixed mask determined by the location of the target image \cite{chen2022survey}. The locations where $x_i$ is placed -- usually the center of $r(x_i)$ -- are denoted as $\{i:M_i=0\}$, which are excluded from further training. The rest of the locations, denoted by $\{i:M_i=1\}$, indicate trainable parameters $\delta$. In watermarking-based methods \cite{bahng2022exploring}, $x_i$ is up-sampled to $r(x_i)$, and $\{i:M_i=1\}$ denotes effective locations of $\delta$ added to $r(x_i)$.

\textbf{Hypothesis Space Induced by SMM}. Based on Eq.~\eqref{eq:final_obj}, we can obtain the hypothesis space used in SMM:
\begin{align*}
    &\mathcal{F}^{\rm smm}({f_{\rm P}^\prime})\\
    =&~\{f|f(x) = {f_{\rm P}^\prime}(r(x)+f_{\rm mask}(r(x))\odot\delta),\forall x\in\mathcal{X}\}.
\end{align*}
Note that, $f_{\rm mask}(r(x))$ belongs to $\mathbb{R}^{d_{\rm P}}$ instead of $\{0, 1\}^{d_\mathcal{\rm P}}$ like $M$. Next, we analyze the relation between the approximation errors of previous VR methods and SMM.

\textbf{SMM Has a Lower Approximation Error.} Based on Theorem~\ref{theorem1} and the two hypothesis spaces above, we have the following proposition.

\begin{proposition} \label{proposition1}
    Given a fixed pre-trained model $f_{\rm P}$, a fixed output mapping function $f_{\rm out}$, and the definitions of $\mathcal{F}^{\rm shr}$ and $\mathcal{F}^{\rm smm}$, we have $\mathcal{F}^{\rm shr}(f_{P}^\prime)\subseteq\mathcal{F}^{\rm smm}(f_{P}^\prime)$. Then, based on Theorem~\ref{theorem1}, we have
\begin{align}
    {\rm Err}^{\rm apx}_{\mathcal{D}_{\rm T}}(\mathcal{F}^{\rm shr}(f_{P}^\prime)) \ge {\rm Err}^{\rm apx}_{\mathcal{D}_{\rm T}}(\mathcal{F}^{\rm smm}(f_{P}^\prime)),
\end{align}
where $f_{P}^\prime=f_{\rm out}\circ f_{\rm P}$, $f_{\rm mask}$ used in $\mathcal{F}^{\rm smm}(f_{P}^\prime)$ is a CNN demonstrated in Section~\ref{sec:generator}, and $\mathcal{D}_{\rm T}$ denotes the distribution of the target task.
\end{proposition}
Proposition \ref{proposition1} (see its proof in Appendix \ref{app:smm_shr}) shows that SMM achieves a lower approximation error than previous shared-mask VR methods.

\textbf{Estimation Error Analysis of SMM.}
While a lower approximation error does not suffice to guarantee a lower excess risk, the model complexity added to $\mathcal{F}^{\rm smm}(f_{P}^\prime)$ is manageable in this VR setting, since $f_{\rm mask}$ introduces less than $0.2\%$ extra parameters\footnote{See Table~\ref{param_statstic} for statistics on network sizes.} relative to $f_{\rm P}^{\prime}$. 
Such dominance of $f_{\rm P}^{\prime}$ suggests that the estimation error of $\mathcal{F}^{\rm smm}(f_{P}^\prime)$ does not significantly exceed that of $\mathcal{F}^{\rm shr}(f_{P}^\prime)$ and is unlikely to offset its advantage in approximation error.
We also provide an empirical justification from the standpoint of over-fitting to show that the additional estimation error of $\mathcal{F}^{\rm smm}(f_{P}^\prime)$ is negligible in Appendix~\ref{app:error}.
By comparing the disparities in training and testing performance, we demonstrate that SMM does not increase the risk of model over-fitting, implying negligible estimation error.

\textbf{Excess Risk Analysis of SMM.}
According to excess risk decomposition\footnote{The excess risk is equal to the sum of approximation error and estimation error \cite{excessrisk}.}, SMM is also expected to have a lower excess risk and, consequently, superior generalization capability compared to shared-mask VR methods.

\textbf{Analysis Based on Sample-specific Patterns.}
Having built the concept of ``sample-specific'', we also investigate an alternative to the proposed SMM: directly learning a \emph{sample-specific pattern} for each image without involving $\delta$. 
The hypothesis space in this context can be expressed by
\begin{align*}\label{eq:ssm_wo_m}
    \mathcal{F}^{\rm sp}({f_{\rm P}^\prime})
    =\{f|f(x) = {f_{\rm P}^\prime}(r(x)+f_{\rm mask}(r(x))),\forall x\in\mathcal{X}\}.
\end{align*}
It is easy to check that $\mathcal{F}^{\rm sp}({f_{ P}^\prime})\subseteq \mathcal{F}^{\rm smm}({f_{ P}^\prime})$, implying that ${\rm Err}^{\rm apx}_{\mathcal{D}_{\rm T}}(\mathcal{F}^{\rm sp}(f_{P}^\prime)) \ge {\rm Err}^{\rm apx}_{\mathcal{D}_{\rm T}}(\mathcal{F}^{\rm smm}(f_{P}^\prime))$ (proof in Appendix \ref{app:smm_sp}). Namely, SMM has a lower approximation error compared to directly learning a sample-specific pattern. 

\section{Experiments} \label{exp}
\textbf{Pre-trained Models and Target Tasks.} Following \citet{chen2023understanding}, we use ResNet-18, and ResNet-50 \cite{he2016deep} as the pre-trained model. Performance on pre-trained ViT-B32 \cite{dosovitskiy2020image} is also tested. All these models are pre-trained on ImageNet-1K \cite{deng2009imagenet}, and target tasks include CIFAR10, CIFAR100 \cite{krizhevsky2009learning}, SVHN \cite{yuval2011reading}, GTSRB \cite{houben2013detection}, Flowers102 \cite{nilsback2008automated}, DTD \cite{cimpoi2014describing}, UCF101 \cite{soomro2012ucf101}, Food101 \cite{bossard2014food}, EuroSAT \cite{helber2019eurosat}, OxfordPets \cite{parkhi2012cats}, SUN397 \cite{xiao2010sun}. Moreover, StanfordCars \cite{krause20133d}, which is revealed to be unsuitable for VR, is also discussed in Appendix \ref{stanfordappend}. We follow \citet{chen2023understanding} to split the datasets. Detailed dataset information is included in Appendix \ref{trainparam}.

\textbf{Baselines.} We compare our method with both padding-based \cite{chen2023understanding} and resizing-based methods \cite{bahng2022exploring}, including:  (1) \emph{Pad}: centering the original image and adding the noise pattern around the images, (2) \emph{Narrow}: adding a narrow padding binary mask with a width of 28 ($\frac{1}{8}$ of the input image size) to the noise pattern that covers the whole image (watermark), (3) \emph{Medium}: adding a mask being a quarter of the size (the width is 56) of watermarks
and (4) \emph{Full}: full watermarks that cover the whole images following \citet{wang2022watermarking}. To ensure that all the methods are fairly compared, in training the shared noise pattern, we apply the same learning rate and milestones following \citet{chen2023understanding}, with 0.01 being the initial learning rate and 0.1 being the learning rate decay. Two hundred epochs are run in total, and the 100$th$ and the 145$th$ epochs are the milestones. The training details of the mask generator are included in Appendix \ref{trainparam}. Experiments are run with three seeds on a single A100 GPU and the averaged test accuracy is reported.
Due to page limits, we report here only the results obtained with the output mapping ${f}_{\rm out}^{\rm Ilm}$.
See Appendix \ref{Asec:diff_f} for the results using ${f}_{\rm out}^{\rm Rlm}$ and ${f}_{\rm out}^{\rm Flm}$.

\begin{figure*}[t]
    \centering
    \includegraphics[width=\linewidth]{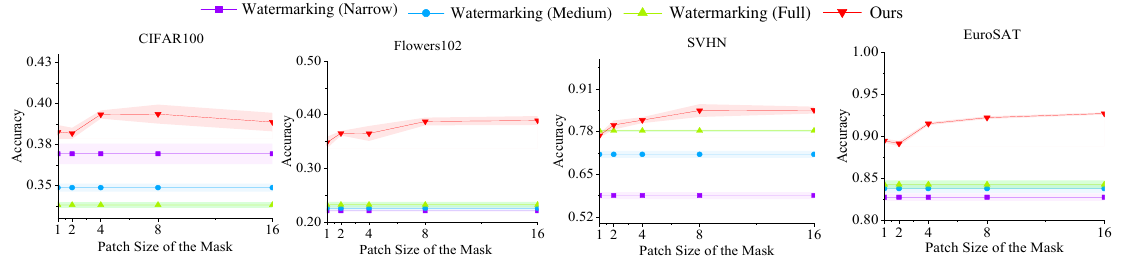}
    \caption{Comparative results of different patch sizes ($2^l$). ResNet-18 is used as the pre-trained model as an example.} 
    \label{fig:patech_size}
\end{figure*}

\textbf{Results on ResNets.} Table \ref{resnetresults} reports the accuracy of ResNet-18 and ResNet-50 using VR methods with the baseline shared marks and our proposed SMM method. It can be observed that our SMM yields higher accuracy for both models on all datasets tested except for ResNet-18 on DTD. The advantage is more pronounced on the datasets where the target domains are more different from the original domain, such as SVHN, Flowers102, and EuroSAT. 
On SVHN, 6.1\% and 3.1\% improvements have been witnessed for ResNet-18 and ResNet-50, respectively, while over 10\% improvement is observed on the Flowers102. 
On DTD, the padding-based method has better results for ResNet-18. This is likely to be due to the noisy watermarks adversely impacting the texture that needs to be classified, leading to the disadvantages of resizing-based methods. 
Even in this challenging setting, our SMM method leads to higher accuracy when applied on the larger pre-trained model ResNet-50. 

\begin{table}[t] 
\caption{Performance Comparison of Different Input Reprogramming Methods on Pre-trained ViT (Mean \%, the average results are highlighted in grey)}
\label{ViTresults}

\begin{center}
\begin{footnotesize}
    
\begin{sc}
\resizebox{0.48\textwidth}{!}{
\begin{tabular}{c|p{1cm}p{1cm}p{1cm}p{1cm}p{1cm}}
\toprule
Pre-trained & \multicolumn{5}{c}{ViT-B32 (ImageNet-1k)}              \\
\midrule
Method     & Pad       & Narrow & Medium & Full & Ours          \\
\midrule
CIFAR10    & 62.4          & 96.6   & 96.5   & 95.8 & \textbf{97.4} \\
CIFAR100   & 31.6          & 74.4   & 75.3   & 75.0 & \textbf{82.6} \\
SVHN       & 80.2          & 85.0   & 87.4   & 87.8 & \textbf{89.7} \\
GTSRB      & 62.3          & 57.8   & 68.6   & 75.5 & \textbf{80.5} \\
Flowers102 & 57.3          & 55.3   & 56.6   & 55.9 & \textbf{79.1} \\
DTD        & 43.7          & 37.3   & 38.5   & 37.7 & \textbf{45.6} \\
UCF101     & 33.6          & 44.5   & \textbf{44.8}   & 40.9 & 42.6 \\
Food101    & 37.4          & 47.3   & 48.6   & 49.4 & \textbf{64.8} \\
SUN397     & 21.8          & 29.0   & 29.4   & 28.8 & \textbf{36.7} \\
EuroSAT    & \textbf{95.9} & 90.9   & 90.9   & 89.1 & 93.5          \\
OxfordPets & 57.6          & 82.5   & 81.0   & 75.3 & \textbf{83.8} \\
\rowcolor{gray!30}
Average & 53.1 & 63.7 & 65.2 & 64.7 & \textbf{72.4} \\
\bottomrule
\end{tabular}}
\end{sc}
\end{footnotesize}
\end{center}
\end{table}

\begin{table}[t] 
\caption{Ablation Studies (Mean \% ± Std \%, with ResNet-18 as an example, and the average results are highlighted in grey)}
\label{ablation}

\begin{center}
\begin{footnotesize}
\begin{sc}
\resizebox{0.47\textwidth}{!}{
\begin{tabular}{c|cccc}
\toprule
\multicolumn{1}{l|}{} & \thead{only \\ $\delta$} & \thead{only \\ $f_{\rm mask}$} & \thead{Single-\\channel \\ $f_{\rm mask}^{\rm s}$} & \thead{Ours}     \\
\midrule
CIFAR10              & 68.9\scriptsize±0.4      & 59.0\scriptsize±1.6            & 72.6\scriptsize±2.6       & \textbf{72.8}\scriptsize±0.7 \\
CIFAR100             & 33.8\scriptsize±0.2      & 32.1\scriptsize±0.3            & 38.0\scriptsize±0.6       & \textbf{39.4}\scriptsize±0.6 \\
SVHN                 & 78.3\scriptsize±0.3      & 51.1\scriptsize±3.1            & 78.4\scriptsize±0.2       & \textbf{84.4}\scriptsize±2.0 \\
GTSRB                & 76.8\scriptsize±0.9      & 55.7\scriptsize±1.2            & 70.7\scriptsize±0.8       & \textbf{80.4}\scriptsize±1.2 \\
Flowers102           & 23.2\scriptsize±0.5      & 32.2\scriptsize±0.4            & 30.2\scriptsize±0.4       & \textbf{38.7}\scriptsize±0.7 \\
DTD                  & 29.0\scriptsize±0.7      & 27.2\scriptsize±0.5            & 32.7\scriptsize±0.5       & \textbf{33.6}\scriptsize±0.4 \\
UCF101               & 24.4\scriptsize±0.9      & 25.7\scriptsize±0.3            & 28.0\scriptsize±0.3       & \textbf{28.7}\scriptsize±0.8 \\
Food101              & 13.2\scriptsize±0.1      & 13.3\scriptsize±0.1            & 15.8\scriptsize±0.1       & \textbf{17.5}\scriptsize±0.1 \\
SUN397               & 13.4\scriptsize±0.2      & 10.5\scriptsize±0.1            & 15.9\scriptsize±0.1       & \textbf{16.0}\scriptsize±0.3 \\
EuroSAT              & 84.3\scriptsize±0.5      & 89.2\scriptsize±0.9            & 90.6\scriptsize±0.5       & \textbf{92.2}\scriptsize±0.2 \\
OxfordPets           & 70.0\scriptsize±0.6      & 72.5\scriptsize±0.3            & 73.8\scriptsize±0.6       & \textbf{74.1}\scriptsize±0.4 \\
\rowcolor{gray!30}
Average & 46.85 & 42.59 & 49.70 & \textbf{52.53} \\
\bottomrule
\end{tabular}}
\end{sc}
\end{footnotesize}
\end{center}
\end{table}

\textbf{Results on ViT.}
Recall that input reprogramming can be applied to diverse pre-trained classifiers, we next turn our focus on ViT. Detailed in Table \ref{ViTresults}, our comparative study with baselines reveals substantial performance gains in datasets like Flowers102~(21.8\%), Food101~(15.4\%), and SUN397~(7.3\%). 
These results suggest that SMM may yield even higher performance gains for larger pre-trained models.
Exceptions do exist, like on EuroSAT, where all resizing-based methods show marginal under-performance, possibly a result of over-fitting on relatively simpler datasets. 
On UCF101, our SMM initially lags behind other strategies like narrow or medium masking but, after choosing appropriate learning rate parameters (See Appendix \ref{trainparam}), could achieve a leading 49.9\% accuracy.
Overall, the experiments above show the applicability of SMM over different pre-trained models and target domains. 
Abnormal cases of SMM in Table~\ref{resnetresults} and Table~\ref{ViTresults} will be further discussed in Appendix~\ref{app:abnormal}.
Next, we report ablation and parameter study results.
\begin{figure*}[ht]
    \centering
    \includegraphics[width=\linewidth]{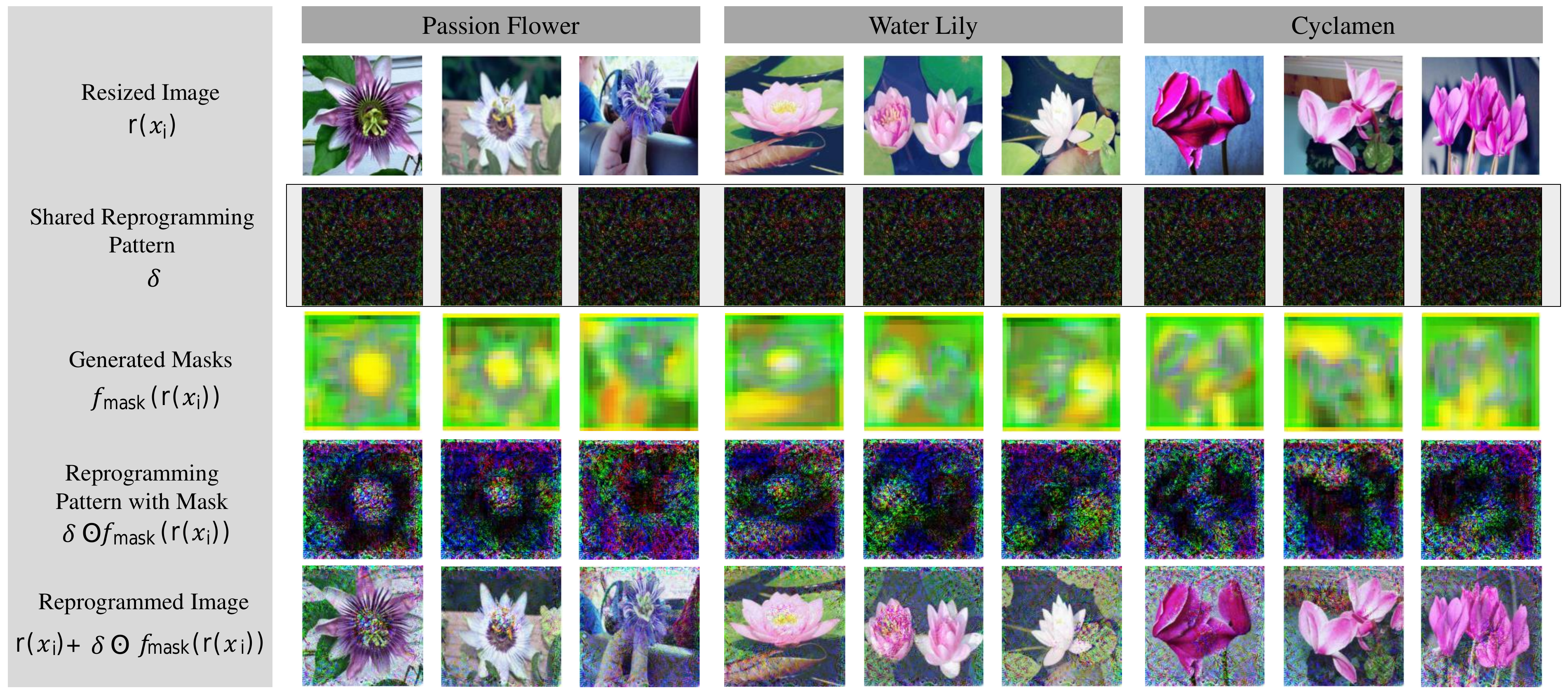}
    \caption{Visual results of trained VR on the Flowers102 dataset. To show the difference in results, the original image, result image and SMM adopt histogram equalization. ResNet-18 is used as the pre-trained model as an example. Other visualization results and further analysis are included in Appendix \ref{visualresults}.}
    \label{fig:visualization}
\end{figure*}

\begin{figure}[!ht]
    \centering
    \includegraphics[width=\linewidth]{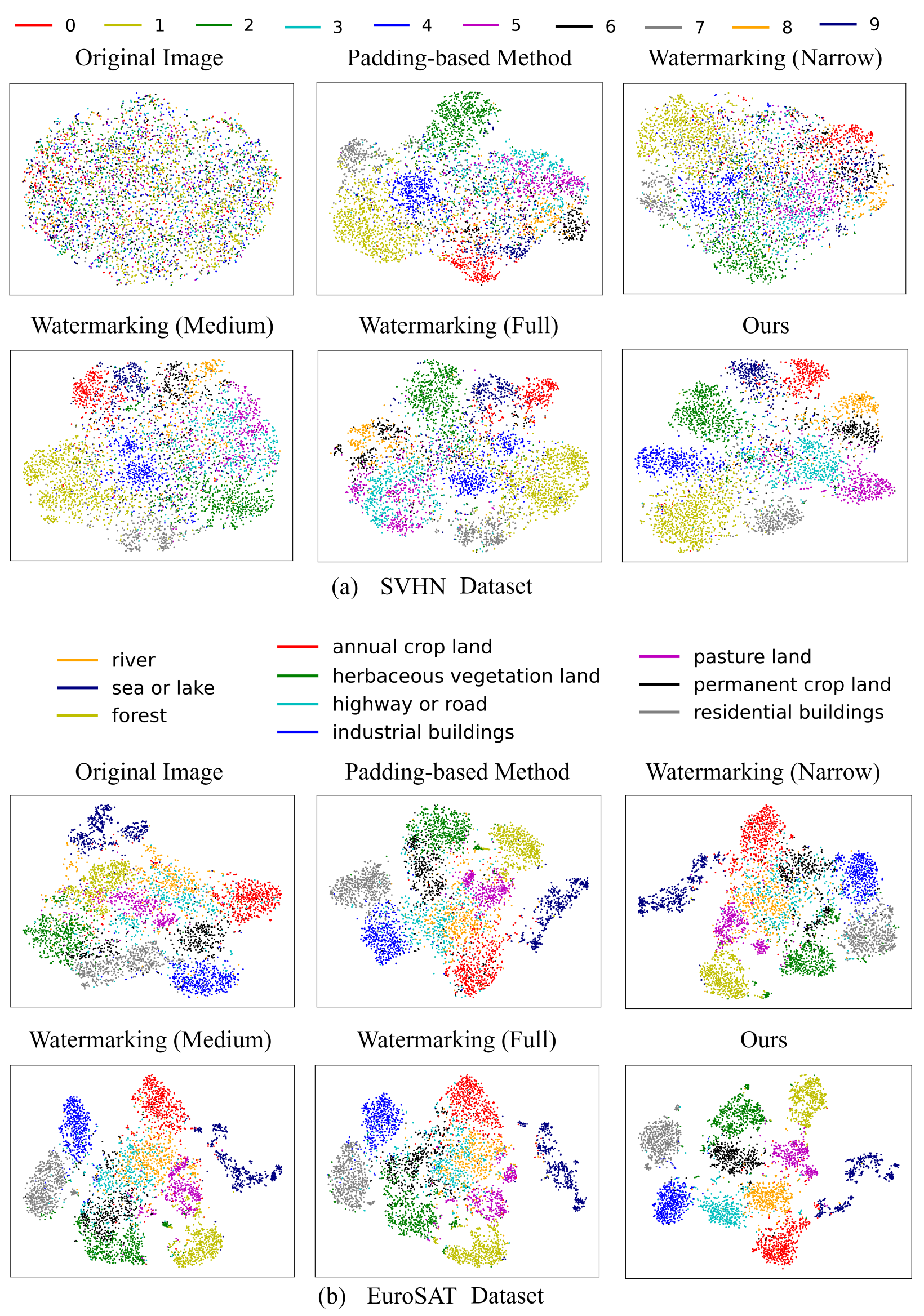}
    \caption{TSNE visualization results of the feature space on (a) SVHN and (b) EuroSAT datasets. ResNet-18 is used as the pre-trained model as an example.} 
    \label{fig:tsne}
\end{figure}
\textbf{Impact of Masking.} \label{sec: ablation}
We first investigate the impact of different masking strategies.
We take three variants against the proposed SMM into comparison:
(i) Shared-pattern VR $f_{\rm in}(x_i)=r(x_i)+\delta$, with $M$ being an all-one matrix equal to the image dimension for maximal flexibility in $\delta$. 
It defaults to the ``full watermarks'' baseline without using $f_{\rm mask}$.
(ii) Sample-specific pattern without masking $f_{\rm in}(x_i)=r(x_i)+f_{\rm mask}(r(x_i))$.
(iii) Single-channel version of SMM $f_{\rm in}(x_i)=r(x_i)+\delta \odot f^{\rm s}_{\rm mask}(r(x_i))$, averaging the penultimate-layer output of the mask generator.
These variants refer to the first three columns of Table \ref{ablation}, respectively.
They help evaluate the impact of sample specificity, masking, and multiple channels introduced by SMM in the context of input VR.

As shown in Table \ref{ablation}, SMM consistently stands out as the best performer on all datasets. A key observation is that only keeping shared pattern $\delta$ reduces VR effectiveness in feature-rich datasets (e.g., CIFAR10, Flowers102, and UCF101). Besides, using only $f_{\rm mask}$ without $\delta$, leads to suboptimal performance on datasets with enough training data per class, including CIFAR10, SVHN, GTSRB, and SUN397. Moreover, the single-channel method is less effective, especially on datasets where images have fewer varying color palettes (e.g., GTSRB and Flowers102). Overall, we find that the shared noise in SMM boosts model performance if sufficient training data is provided, whereas the sample-specific $f_{\rm mask}$ enables specificity for classification tasks demanding detailed feature discrimination. Lastly, the multi-channel allows for adjusting to channel-specific priorities.

\textbf{Impact of Patch Size.}
As an important hyperparameter in SMM, number of Max-Pooling layers, $l$, can vary, which means different patch sizes $2^l$. Since the 5-layer mask generator neural network has at most 4 Max-Pooling layers, we examine the impact of patch sizes in $\{2^0$, $2^1$, $2^2$, $2^3$, $2^4\}$. Results are shown in Figure \ref{fig:patech_size}. As the patch size increases, the accuracy of the SMM increases first, followed by a plateau or decline. This suggests that overly small patches may cause over-fitting, while overly large patch sizes could result in a loss of details in SMM. We thus have set the patch size to be 8 across all datasets.

\textbf{Visualization of SMM, shared patterns and output reprogrammed images.} Visualization results on Flowers102 dataset is shown in Figure \ref{fig:visualization}. It can be observed that when classifying passion flowers, where pedals are important for classification accuracy, the masks tend to mask out the noise pattern over the pedals, which protects useful information from being shadowed by noise. Other features such as flower pistils in passion flowers are also widely present in various similar classes such as ‘oxeye’, ‘daisy’ and ‘orange dahlia’, making the centers of flowers potential sources of interference in classification. Thus, for passion flowers, noise in the center of the flowers is not masked out. When classifying `water lily', SMM will enhance the noise on interfering objects in the image. Similarly, when classifying ‘cyclamen’, similar stems are also commonly found in other classes such as ‘gaura’ and ‘rose’, which hinders accurate classification. Therefore, it is reasonable for SMM to introduce more noise to these interfering components. These results show that SMM is able to retain the important parts of the image and remove the interference.

\textbf{Feature Space Visualization Results.} Figure \ref{fig:tsne} shows the tSNE \cite{van2008visualizing} visualization results of the output layer feature before the label mapping layer. Before applying VR methods, the target domain's output feature space shows limited class separation. With the baseline methods, we observe enhanced but incomplete separations, where certain class pairs (such as `3, 5' and `6, 8' in SVHN, `River’ and `highway or road' in EuroSAT) remain indistinguishable in the feature space.
By applying $f_{\rm mask}$, our method successfully resolves incorrectly clustered classes, underscoring the effectiveness of SMM.

\textbf{Comparison with Finetuning-based Methods.}
In Appendix~\ref{app:compare}, we compare our SMM with two prevalent finetuning approaches: finetuning fully connected layers and low-rank adaptation \cite{zhu2023melo}. This comparison highlights two key benefits of input VR: (1) its efficacy in target tasks with lower-resolution images and (2) its orthogonal relationship to, yet compatibility with, finetuning methods. Additionally, Appendix~\ref{app:compare} provides a comprehensive discussion on the strengths and weaknesses of Input VR in comparison to finetuning techniques.

\textbf{More Experiments}. The training curves are plotted and analyzed in Appendix \ref{trainingcurves}. The effectiveness of SMM when learning with different $f_{\rm out}$ is discussed in Appendix~\ref{Asec:diff_f}. 

\section{Conclusion}
In this paper, we identified significant shortcomings in the use of a shared mask across all samples in previous VR practices, notably its failure to accommodate sample diversity, leading to increased training loss of particular samples.
In response, we proposed a new SMM learning framework, integrating a lightweight neural net-based mask generator to generate three-channel masks per sample, and a patch-wise interpolation module that resizes and aligns masks to model input. Both theoretical justification and experimental results validated the effectiveness of our proposed method.

\section*{Acknowledgements}

CYC and FL are supported by the Australian Research Council (ARC) with grant number DE240101089, and FL is also supported by the ARC with grant number DP230101540 and the NSF\&CSIRO Responsible AI program with grant number 2303037. JZQ is supported by ARC with grant number DP240101006. This research is also supported by The University of Melbourne’s Research Computing Services and the Petascale Campus Initiative. We sincerely appreciate the time and dedication of the reviewers in carefully reviewing our manuscript.

\section*{Impact Statement}
This paper presents work whose goal is to advance the field of 
Machine Learning. There are many potential societal consequences 
of our work, none of which we feel must be specifically highlighted here.

\nocite{langley00}

\bibliography{example_paper}
\bibliographystyle{icml2024}

\newpage
\appendix
\onecolumn
\section*{Appendix}
\section{Additional Explanation of Methods}
\subsection{General Procedure of Input Visual Reprogramming} \label{prob}
\begin{figure}[ht]
    \centering
    \includegraphics[width=0.7\linewidth]{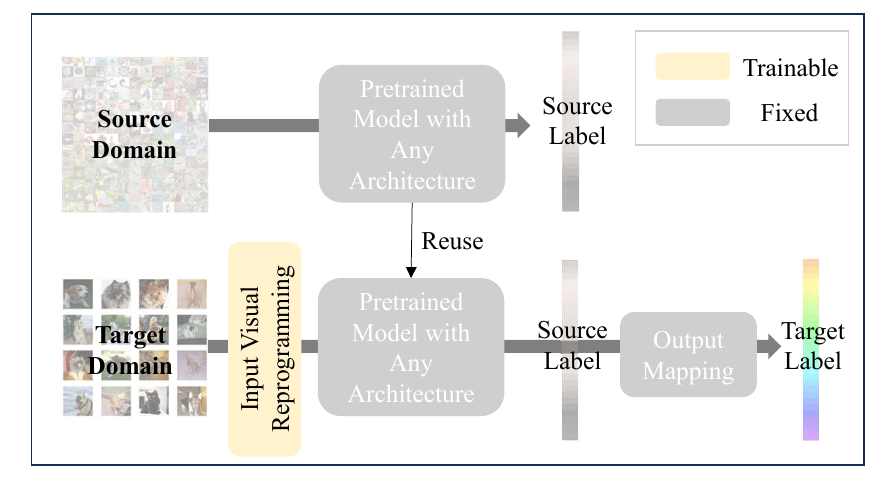}
    \caption{Problem setting of input visual reprogramming. The upper part shows the source task, while the lower part shows the target task. The main focus of visual reprogramming is the trainable part marked with a yellow rectangle in the input space.}
    \label{fig:problem setting}
\end{figure}
The task of VR is to reuse the fixed, well-trained model toward a target task. As shown in Figure \ref{fig:problem setting}, the VR module is added before the pre-trained model into the input space. To gap the difference between the source label and target label, an output mapping function without parameters is also used, taking a source label as the input and outputting a target label. Therefore, regardless of the architecture, a well-trained model on the source dataset can be transferred to the target task without editing. 

\subsection{Architecture of the Mask Generator and Parameter Statistics} \label{architecture}

\begin{figure}[ht]
    \centering
    \includegraphics[width=0.8\linewidth]{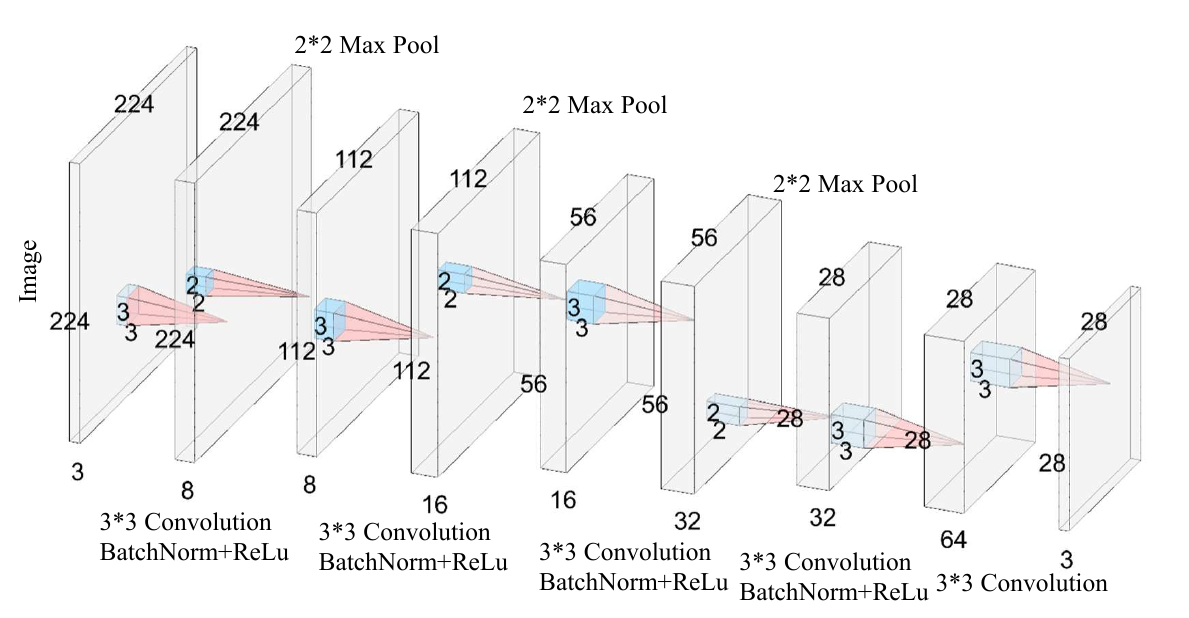}
    \caption{Architecture of the 5-layer mask generator designed for ResNet}
    \label{fig:5layerarchitecture}
\end{figure}

\begin{figure}[ht]
    \centering
    \includegraphics[width=0.8\linewidth]{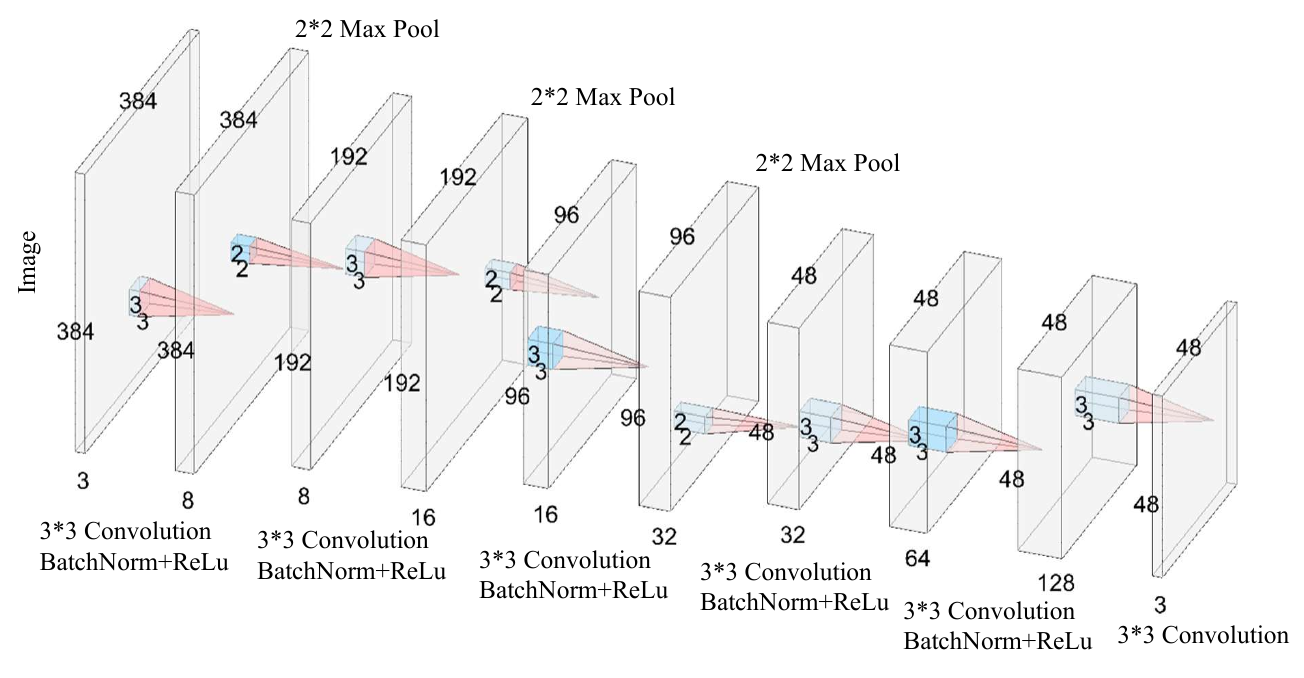}
    \caption{Architecture of the 6-layer mask generator designed for ViT}
    \label{fig:6layerarchitecture}
\end{figure}

\textbf{Architecture of the Mask Generator.}
For simplicity, we only include $3\times 3$ convolution layers and $2 \times 2$ Max-Pooling layers in the architecture. 
The number of channels of the last layer is set to 3 to produce a three-channel mask.

The detailed architecture of the 5-layer CNN and 6-layer CNN used in ResNet-18, ResNet-50, and ViT are shown in Figure \ref{fig:5layerarchitecture} and Figure \ref{fig:6layerarchitecture}. Each of them contains 5 or 6 CNN layers with $3 \times 3$ kernels of padding size 1 and stride 1. Both models have 3 Max-Pooling layers. 

\begin{figure}[ht]
    \centering
    \includegraphics[width=0.6\linewidth]{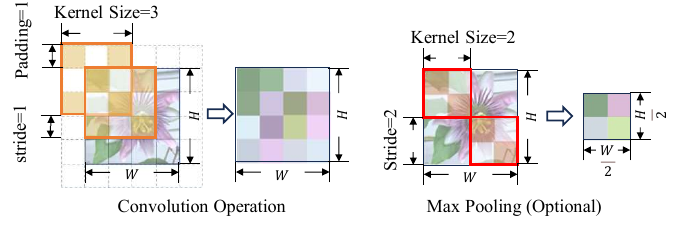}
    \caption{Changes of the image size when performing convolution and pooling operations with our stride, kernel and padding size}
    \label{fig:Convolution}
\end{figure}

\textbf{Discussion of Input and Output Size.}
To show the relationship between the sizes of the input images and the output masks, we use $s$, $p$, and $k$ to represent the stride, padding, and kernel sizes, respectively, while $H$ and $W$ denote the height and the width of a certain channel. The output dimensions of the output channel after convolution or pooling are $\left \lfloor \frac{H+2p-k}{s}  \right \rfloor +1$ and $\left \lfloor \frac{W+2p-k}{s}  \right \rfloor +1$. As shown in Figure \ref{fig:Convolution}, when $s=1, p=1, k=3$, the size of a single channel remains unchanged; when $s=2, p=0, k=2$, the size of a channel is reduced by half in each dimension.
In other words, by only using $3 \times 3$ convolution layers, $f_{\rm mask}(.|\phi)$ can retain the original size of a single channel. However, if we introduce Max-Pooling layers to remove redundant information, the output size will be shrunk and another patch-wise interpolation module should be included in $f_{\rm mask}(.|\phi)$ for resizing. Assuming that $l$ Max-Pooling layers are used, the output size of a single channel becomes $\left \lfloor  \frac{H}{2^l}\right \rfloor  \times \left \lfloor \frac{W}{2^l}\right \rfloor $.

\begin{table}[ht]
\caption{Statistics of Mask Generator Parameter Size}
\label{param_statstic}

\begin{center}
\begin{small}
\begin{sc}
\begin{tabular}{cccccc}
\toprule
Pre-trained & \makecell{Input \\ Image \\ Size} & \makecell{$f_{\rm mask}$ \\ CNN \\ Layers} & \makecell{Extra Parameters \\ of our $f_{\rm mask}$} & \makecell{Our Extra \\ Parameters $\div$ \\ Reprogramming \\ Parameters} & \makecell{Our Extra \\ Parameters $\div$ \\ Pre-trained \\ Model Parameters} \\
\midrule
ResNet-18  & 224$\times$224$\times$2  & 5          & 26,499                           & 17.60\%                                           & 0.23\%                                               \\
ResNet-50  & 224$\times$224$\times$3  & 5          & 26,499                           & 17.60\%                                           & 0.10\%                                               \\
ViT-B32    & 384$\times$384$\times$3  & 6          & 102,339                          & 23.13\%                                           & 0.12\%    \\
\bottomrule
\end{tabular}
\end{sc}
\end{small}
\end{center}
\end{table}
\textbf{Parameter Statistics.}
The parameter statistics of the mask generator, $f_{\rm mask}$, are summarized in Table \ref{param_statstic}. 
This includes a detailed breakdown of $f_{\rm mask}$ across different pre-trained backbone models, a relative size comparison with the watermarking reprogramming method, and the number of trainable parameters added to frozen pre-trained models by $f_{\rm mask}$.
From the size point of view, our mask generator is indeed lightweight and efficient: 
the CNN architectures contribute only 17.6\% and 23.13\% of the additional trainable parameters required by watermarking reprogramming. 
Moreover, relative to the total parameters in pre-trained models, the additional contribution of mask generators is trivial, ranging from 0.1\% to 0.23\% of parameters, which highlights its minimal footprint. 

\subsection{Advantage of Patch-wise Interpolation}
\label{app:interpolation}
\begin{table}[ht]
\caption{Comparison of Patch-wise Interpolation and Other Interpolation Methods}

\begin{center}
\begin{small}
\begin{sc}
\begin{tabular}{c|l|lll}
\toprule
\multicolumn{1}{l}{}            &                                & \makecell{Bilinear \\ Interpolation} & \makecell{Bicubic \\ Interpolation} & Ours                 \\
\midrule
\multirow{3}{*}{ResNet - 18/50} & \makecell{Number of \\ Pixel \\ Accesses (1e6)}    & 0.602                  & 2.408                 & \textbf{0.151}       \\
                                & \makecell{Time Per \\ Batch (s)}             & 0.062±0.001            & 0.195±0.013           & \textbf{0.026±0.004} \\
                                & \makecell{Require \\ Backpropagation}       & Yes                    & Yes                   & \textbf{No}          \\
\midrule
\multirow{3}{*}{ViT - B32}      & \makecell{Number of \\ Pixel \\ Accesses (1e6)} & 1.769                  & 7.078                 & \textbf{0.442}       \\
                                & \makecell{Time Per \\ Batch (s)}             & 0.165±0.009            & 0.486±0.026           & \textbf{0.069±0.004} \\
                                & \makecell{Require \\ Backpropagation}       & Yes                    & Yes                   & \textbf{No} \\
\bottomrule
\end{tabular}
\end{sc}
\end{small}
\end{center}
\label{tab:interpolation}
\end{table}

To assess the efficiency of patch-wise interpolation, we compare it with bilinear and bicubic methods, employing the following numerical metrics for evaluation:
(1) Number of Pixel Accesses: The count of times pixel values are retrieved per image during an interpolation algorithm. The fewer, the better.
(2) Time Per Batch: The time cost for processing a batch of 256-sized images. The fewer, the better.

As shown in Table \ref{tab:interpolation}, the patch-wise interpolation module excels across all metrics. This module exclusively involves copying operations, thus avoiding floating-point calculations and avoiding backpropagation gradient computations during training. Consequently, it is more efficient.

\subsection{Detailed Explanation of Ouptput Mapping Methods $f^{\rm Flm}_{\rm out}$ and $f^{\rm Ilm}_{\rm out}$}\label{app:out}
The inverse function of $f_{\rm out}$ regarding Flm is an injective function: 
\begin{align}
y^{\rm P}_{\rm Flm} = \mathop{\arg \max}\limits_{y\in\mathcal{Y}^{\rm P}}\mathop{Pr}\limits_{(x_i,y_i)\sim \mathcal{D_{\rm  T}}}
\{y = f_{\rm P}(f_{\rm in}(x_i|\theta))| y_i=y^{\rm T}\},
\end{align}
where $y^{\rm P}_{\rm Flm}$ is the optimal $y^{\rm P}$ given the target label $y^{\rm T}$, $f_{\rm P}(f_{\rm in}(x_i|\theta))$ is the predicted label given the input image $x_i$. For all images with the label $y^{\rm T}$, the predicted $y^{\rm P}$ with the highest probability will be $y^{\rm P}_{\rm Flm}$ for a given $y^{\rm T}$.
Flm remains unchanged throughout iterations. For a specific $y^{\rm T}$, Flm determines the correspondence between $y^{\rm T}$ and the most frequently assigned class $y^{\rm P}$ in $\mathcal{Y}^{\rm P}$, utilizing the well-trained network for all target training samples of the class $y^{\rm T}$, thus obtaining $f_{\rm out}^{\rm Flm}$, shown in Algorithm~\ref{alg:flm}.

As the label mapping may change from time to time when learning $f_{\rm in}$, \citet{chen2023understanding} proposed an \emph{iterative label mapping} (Ilm) method that updates $f_{\rm out}(\cdot)$ after each training iteration. Let $y^{\rm P,(j)}_{\rm Ilm}$ be the optimal $y^{\rm P}$ in the $j$th training epoch. We have:
\begin{align}
y^{\rm P,(j+1)}_{\rm Ilm} & 
= \mathop{\arg \max}\limits_{y\in\mathcal{Y}^{\rm P}}\mathop{Pr}\limits_{(x_i,y_i)\sim \mathcal{D_{\rm  T}}}\{y = f_{\rm P}(f_{\rm in}^{(j)}(x_i|\theta^{(j)}))|y_i=y^{\rm T}\},
\end{align}
where $f_{\rm in}^{(j)}(\cdot|\theta^{(j)})$ is the parameters of the $j$th epoch. The output mapping function is updated after each iteration until convergence.

\begin{algorithm}[h]
\caption{Computing Frequency Distribution of [$f_{\rm P}(f_{\rm in}(x_i|\theta))$, $y^{\rm T}$]}
\label{alg:dm}
\begin{algorithmic}[1] 
\STATE {\bfseries Input:} Target training set $\{(x^{\rm T}_i, y^{\rm T}_i)\}_{i=1}^{n}$, given input VR $f_{\rm in}(\cdot|\theta)$ and pre-trained model $f_{\rm P}(\cdot)$
\STATE {\bfseries Output:} Frequency distribution matrix $d \in \mathbb{Z}^{|\mathcal{Y}^{\rm P}| \times |\mathcal{Y}^{\rm T}|}$
\STATE Initialize $d \leftarrow \{0\}^{|\mathcal{Y}^{\rm P}| \times |\mathcal{Y}^{\rm T}|}$
\STATE \text{\# Compute frequency distribution $d$}
\FOR{$i=1...n$}
\STATE $\hat{y}_i^{\rm P} \leftarrow f_{\rm P}(f_{\rm in}(x^{\rm T}_i|\theta))$
\STATE $d_{\hat{y}_i^{\rm P},{y}_i^{\rm T}}\leftarrow d_{\hat{y}_i^{\rm P},{y}_i^{\rm T}}+1$
\ENDFOR
\end{algorithmic}
\end{algorithm}

\begin{algorithm}[h]
\caption{Frequent Label Mapping ($f^{\rm Flm}_{\rm out}$)}
\label{alg:flm}
\begin{algorithmic}[1] 
\STATE {\bfseries Input:} Label space of the pre-trained task $\mathcal{Y}^{\rm P}$, label space of the target task $\mathcal{Y}^{\rm T}$, target training set $\{(x^{\rm T}_i, y^{\rm T}_i)\}_{i=1}^{n}$, given pre-trained model $f_{\rm P}(\cdot)$
\STATE {\bfseries Output:} Flm $f^{\rm Flm}_{\rm out}:\mathcal{Y}^{\rm P}_{\rm sub} \rightarrow \mathcal{Y}^{\rm T}$
\STATE Initialize $f^{\rm Flm}_{\rm out}(\cdot) \leftarrow 0$, subset $\mathcal{Y}^{\rm P}_{\rm sub} \leftarrow \emptyset$ to store matched labels, initialize $f_{\rm in}(\cdot|\theta)$ to be an identity function ($\theta \leftarrow \mathbf{0}$)
\STATE \text{\# Compute frequency distribution $d$}
\STATE Use Algorithm \ref{alg:dm} to obtain $d$
\STATE \text{\# Compute output mapping $f^{\rm Flm}_{\rm out}$}
\WHILE{size of $\mathcal{Y}^{\rm P}_{\rm sub}$ is not $|\mathcal{Y}^{\rm T}|$}
\STATE Find the maximum $d_{y^{\rm P}, y^{\rm T}}$ in $d$
\STATE $\mathcal{Y}^{\rm P}_{\rm sub} \leftarrow \mathcal{Y}^{\rm P}_{\rm sub} \cup \{y^{\rm P}\}$
\STATE $f^{\rm Flm}_{\rm out}(y^{\rm P}) \leftarrow y^{\rm T}$ \text{ \# Update the label mapping function}
\STATE $d_{y^{\rm P},t}\leftarrow 0$ for $t=1,2,...,|\mathcal{Y}^{\rm T}|$ \text{ \# Avoiding illegal assignment to the injective function}
\STATE $d_{s,y^{\rm T}}\leftarrow 0$ for $s=1,2,...,|\mathcal{Y}^{\rm P}|$

\ENDWHILE
\end{algorithmic}
\end{algorithm}

\begin{algorithm}[h]
\caption{Iterative Label Mapping ($f^{\rm Ilm}_{\rm out}$)}
\label{alg:ilm}
\begin{algorithmic}[1] 
\STATE {\bfseries Input:} Label space of the pre-trained task $\mathcal{Y}^{\rm P}$, label space of the target task $\mathcal{Y}^{\rm T}$, target training set $\{(x^{\rm T}_i, y^{\rm T}_i)\}_{i=1}^{n}$, given pre-trained model $f_{\rm P}(\cdot)$, total iteration number $E$, learning rate $\alpha$
\STATE {\bfseries Output:} Ilm $f^{\rm Ilm, (j)}_{\rm out}:\mathcal{Y}^{\rm P}_{\rm sub} \rightarrow \mathcal{Y}^{\rm T}$ for iteration $j$
\STATE Initialize $f^{\rm Ilm, (j)}_{\rm out}(\cdot) \leftarrow 0$, subset $\mathcal{Y}^{\rm P}_{\rm sub} \leftarrow \emptyset$ to store matched labels, initialize $f_{\rm in}(\cdot|\theta)$ to be an identity function ($\theta \leftarrow \mathbf{0}$)
\FOR{$j=1...E$}
\STATE \text{\# Compute frequency distribution $d$}
\STATE Use Algorithm \ref{alg:dm} to obtain $d$
\STATE \text{\# Compute output mapping $f^{\rm Ilm, (j)}_{\rm out}$}
\WHILE{size of $\mathcal{Y}^{\rm P}_{\rm sub}$ is not $|\mathcal{Y}^{\rm T}|$}
\STATE Find the maximum $d_{y^{\rm P}, y^{\rm T}}$ in $d$
\STATE $\mathcal{Y}^{\rm P}_{\rm sub} \leftarrow \mathcal{Y}^{\rm P}_{\rm sub} \cup \{y^{\rm P}\}$
\STATE $f^{\rm Ilm, (j)}_{\rm out}(y^{\rm P}) \leftarrow y^{\rm T}$  \text{ \# Update the label mapping function for iteration $j$}
\STATE $d_{y^{\rm P},t}\leftarrow 0$ for $t=1,2,...,|\mathcal{Y}^{\rm T}|$  \text{ \# Avoiding illegal assignment to the injective function}
\STATE $d_{s,y^{\rm T}}\leftarrow 0$ for $s=1,2,...,|\mathcal{Y}^{\rm P}|$
\ENDWHILE
\STATE \text{\# Train $f_{\rm in}(\cdot|\theta)$ for iteration $j$}
\STATE $\theta \leftarrow \theta - \alpha\cdot \nabla_{\theta}\frac{1}{n}\sum_{i=1}^n\ell(f^{\rm Ilm, (j)}_{\rm out}(f_{\rm P}(f_{\rm in}(x^{\rm T}_i|\theta))), y^{\rm T}_i)$
\ENDFOR
\end{algorithmic}
\end{algorithm}

Ilm evolves with iterations, being an improved version of Flm. As is shown in Algorithm~\ref{alg:ilm}, before training the reprogramming pattern $\theta$ in each epoch, Ilm updates the one-to-one mapping from $\mathcal{Y}^{\rm P}$ to $\mathcal{Y}^{\rm T}$ with the training samples incorporating the current pattern, iteratively until convergence.

\section{Additional Theoretical Proof} \label{proof}
\subsection{Proof of Theorem \ref{theorem1}}
\label{app:approx}
The approximation error of $\mathcal{F}_1$ and $\mathcal{F}_2$ can be formulated as:
\begin{align}
    \nonumber
    & {\rm Err}^{\rm apx}_{\mathcal{D}}(\mathcal{F}_1) = \inf_{f\in\mathcal{F}_1}\mathbb{E}_{(X,Y)\sim\mathcal{D}}\ell(f(X),Y) - R^*_{\mathcal{D}}, \\
    \nonumber
    & {\rm Err}^{\rm apx}_{\mathcal{D}}(\mathcal{F}_2) = \inf_{f\in\mathcal{F}_2}\mathbb{E}_{(X,Y)\sim\mathcal{D}}\ell(f(X),Y) - R^*_{\mathcal{D}},
\end{align}
Straightforwardly, 
\begin{align}
     \nonumber
     & \mathcal{F}_1 \supseteq \mathcal{F}_2 \Leftrightarrow \forall f \in \mathcal{F}_2, f \in \mathcal{F}_1 
\end{align}
Given $\mathcal{F}_1\subseteq\mathcal{F}_2$, we have:
\begin{align}
    \nonumber
    & \forall f \in \mathcal{F}_1, f \in \mathcal{F}_2, \\
    \nonumber
    \Rightarrow & \inf_{f\in\mathcal{F}_1}\mathbb{E}_{(X,Y)\sim\mathcal{D}}\ell(f(X),Y) \ge \inf_{f\in\mathcal{F}_2}\mathbb{E}_{(X,Y)\sim\mathcal{D}}\ell(f(X),Y) \\
    \nonumber
     \Rightarrow & {\rm Err}^{\rm apx}_{\mathcal{D}}(\mathcal{F}_1) \ge {\rm Err}^{\rm apx}_{\mathcal{D}}(\mathcal{F}_2)
\end{align}

\subsection{Proof of Proposition \ref{proposition1}}
\label{app:smm_shr}
We prove Proposition~\ref{proposition1} as follows.
\begin{proof}
\renewcommand{\qedsymbol}{}
    With specially designed kernel and padding sizes, the output of CNN can be reshaped to match the size of the input image. Assuming $d_{\rm P} = H \times W \times C$, we define $M'\in \{0 ,1\}^{H*W*C\times 1}$ and $f'_{\rm mask}(\cdot)\in \mathbb{R}^{H*W*C\times 1}$ as transposed flattened $M$ and $f_{\rm mask}(\cdot)$, respectively. As the last layer of $f'_{\rm mask}(\cdot)$ is CNN, if the input of CNN is the resized image $r(x)$, with $x \in \mathcal{X}^{\rm T}$ (and $r(x) \in \mathbb{R}^{d_{\rm P}}$), we have $f'_{\rm mask}(r(x)) = W_{\rm last}f''_{\rm mask}(r(x))+b_{\rm last}$, with $b_{\rm last}$ being the bias of the last layer, and $W_{\rm last}$ being the mapping from the flattened input of the last CNN layer (i.e., $f''_{\rm mask}(r(x))$) to the flattened output without adding the bias, which can be derived using the parameters of the last CNN layer. With the set of any possible $W_{\rm last}$ being represented by $\{W_{\rm last}\}$, and all-zero matrix being $O$, we have:
\begin{align}
    \nonumber
    & b_{\rm last}\in \mathbb{R}^{H*W*C\times 1}, M'\in \{0 ,1\}^{H*W*C\times 1} \nonumber \\
    & \Rightarrow \forall M', M' \in \{b_{\rm last}| b_{\rm last}\in \mathbb{R}^{H*W*C\times 1}\} \label{eq_part1}\\
    \nonumber
    & O \in \{W_{\rm last}\} \text{\scriptsize     (When all weights in the last CNN layer is 0, $W_{\rm last}$ is a zero matrix)} \\ 
    & \Rightarrow f(x)= O^{H*W*C\times 1} \in \{f|f(x)=W_{\rm last}f''_{\rm mask}(r(x)), \forall x \in \mathcal{X}^{\rm T}\} \label{eq_part2} \\
    \nonumber
    & \Rightarrow \{f|f(x)=M', \forall x \in \mathcal{X}^{\rm T}\} \subseteq \{f|f(x)=f'_{\rm mask}(r(x)), \forall x \in \mathcal{X}^{\rm T}\} \text{\scriptsize     (Given Eq.~\eqref{eq_part1} and Eq.~\eqref{eq_part2})}\\
    \nonumber
    & \Rightarrow \{f|f(x)=M, \forall x \in \mathcal{X}^{\rm T}\} \subseteq \{f|f(x)=f_{\rm mask}(r(x)), \forall x \in \mathcal{X}^{\rm T}\} \\
    \nonumber
    & \Rightarrow \{f|f(x)=M \odot \delta, \forall x \in \mathcal{X}^{\rm T}\} \subseteq \{f|f(x)=f_{\rm mask}(r(x))  \odot \delta, \forall x \in \mathcal{X}^{\rm T}\} \\
    \nonumber
     & \Rightarrow \mathcal{F}^{\rm shr}(f'_{\rm P})\subseteq \mathcal{F}^{\rm smm}(f'_{\rm P})  \text{\scriptsize  (since $f'_{\rm P}$ is fixed)} \\
     \nonumber
     & \Rightarrow {\rm Err}^{\rm apx}_{\mathcal{D}_{\rm T}}(\mathcal{F}^{\rm smm}(f'_{\rm P})) \le {\rm Err}^{\rm apx}_{\mathcal{D}_{\rm T}}(\mathcal{F}^{\rm shr}(f'_{\rm P}))
\end{align}
\end{proof}

\subsection{SMM and Sample-specific Patterns}
\label{app:smm_sp}
We will then prove
\begin{proposition}
    for any fixed $f'_{\rm P}$, it holds that $\mathcal{F}^{\rm sp}(f'_{\rm P}) \subseteq \mathcal{F}^{\rm smm}(f'_{\rm P})$, and consequently, ${\rm Err}^{\rm apx}_{\mathcal{D}_{\rm T}}(\mathcal{F}^{\rm smm}(f'_{\rm P})) \le {\rm Err}^{\rm apx}_{\mathcal{D}_{\rm T}}(\mathcal{F}^{\rm sp}(f'_{\rm P}))$.    
\end{proposition}

\begin{proof}
\renewcommand{\qedsymbol}{}
    Let $\Delta$ be the set of possible $\delta$, with all-one matrix being denoted as $J$, we have:
\begin{align}
     \nonumber
     & \Rightarrow J^{d_{\rm P}} \in \Delta \\
     \nonumber
     & \Rightarrow \{f|f(x)=f_{\rm mask}(r(x)) \odot J^{d_{\rm P}}, \forall x \in \mathcal{X}^{\rm T}\} \subseteq \{f|f(x)=f_{\rm mask}(r(x)) \odot \delta, \forall x \in \mathcal{X}^{\rm T}\} \\
     \nonumber
     & \Rightarrow \{f|f(x)=f_{\rm mask}(r(x)), \forall x \in \mathcal{X}^{\rm T}\} \subseteq \{f|f(x)=f_{\rm mask}(r(x)) \odot \delta, \forall x \in \mathcal{X}^{\rm T}\} \\
     \nonumber
     & \Rightarrow \mathcal{F}^{\rm sp}(f'_{\rm P}) \subseteq \mathcal{F}^{\rm smm}(f'_{\rm P})  \text{\scriptsize  (Since $f'_{\rm P}$ is fixed)} \\
     \nonumber
     & \Rightarrow {\rm Err}^{\rm apx}_{\mathcal{D}_{\rm T}}(\mathcal{F}^{\rm smm}(f'_{\rm P})) \le {\rm Err}^{\rm apx}_{\mathcal{D}_{\rm T}}(\mathcal{F}^{\rm sp}(f'_{\rm P}))
\end{align}
\end{proof}

\section{Additional Experimental Setup} \label{trainparam}

\begin{table}[ht]
\caption{Detailed Dataset Information}
\label{datasets}

\begin{center}
\begin{small}
\begin{sc}
\begin{tabular}{c|cccc}
\toprule
\multicolumn{1}{c|}{Dataset} & Original Image Size & Training Set Size & Testing Set Size & Number of Classes \\
\midrule
CIFAR10                     & 32 $\times$ 32      & 50000      & 10000     & 10           \\
CIFAR100                    & 32 $\times$ 32      & 50000      & 10000     & 100          \\
SVHN                        & 32 $\times$ 32      & 73257      & 26032     & 10           \\
GTSRB                       & 32 $\times$ 32      & 39209      & 12630     & 43           \\
Flowers102                  & 128 $\times$ 128    & 4093       & 2463      & 102          \\
DTD                         & 128 $\times$ 128    & 2820       & 1692      & 47           \\
UCF101                      & 128 $\times$ 128    & 7639       & 3783      & 101          \\
Food101                     & 128 $\times$ 128    & 50500      & 30300     & 101          \\
SUN397                      & 128 $\times$ 128    & 15888      & 19850     & 397          \\
EuroSAT                     & 128 $\times$ 128    & 13500      & 8100      & 10           \\
OxfordPets                  & 128 $\times$ 128    & 2944       & 3669      & 37      \\
\bottomrule
\end{tabular}
\end{sc}
\end{small}
\end{center}
\end{table}

The 11 datasets used for the experiments are summarized in Table \ref{datasets}, while the corresponding training parameters are listed in Table \ref{training}. When learning the ResNet tasks, we follow the same learning strategies as \citet{chen2023understanding}. When learning ViT-B32, we choose the initial learning rate $\alpha$ and the learning rate decay $\gamma$ with a training parameter searching experiment, with results presented in Table \ref{hyperparam}.

\begin{table}[ht]
\caption{Tuning Initial Learning Rate and Learning Rate Decay Using CIFAR10 and ViT-B32 (Accucracy \%)}
\label{hyperparam}

\begin{center}
\begin{small}
\begin{sc}
\begin{tabular}{c|cccc}
\toprule
$\gamma | \alpha$ & 0.1    & 0.01   & 0.001           & 0.0001 \\
\midrule
1   & 0.9542 & 0.9577 & \textbf{0.9745} & 0.9734 \\
0.1 & 0.9516 & 0.9572 & 0.9738          & 0.9727 \\
\bottomrule
\end{tabular}
\end{sc}
\end{small}
\end{center}
\end{table}

Sharing the same $\alpha$ and $\gamma$ may not be optimal for all datasets. As shown in Table~\ref{EuroHyper}, on UCF101, using $\alpha=0.001$ and $\gamma = 1$ derived from Table 7 leads to sub-optimal model performance. Nevertheless, for uniformity and fairness in this paper, we still use a single set of unified training parameters for all datasets.

\begin{table}[ht]
\caption{Results on UCF101 with Different Training Parameters (using ViT-B32)}
\label{EuroHyper}

\begin{center}
\begin{small}
\begin{sc}
\begin{tabular}{cccc}
\toprule
                             & $\alpha$ & $\gamma$ & SMM Accuracy (\%) \\
\midrule
Unified Learning Parameters  & 0.001    & 1        & 42.6              \\
Specific Learning Parameters & 0.01     & 0.1      & 49.9  \\
\bottomrule
\end{tabular}
\end{sc}
\end{small}
\end{center}
\end{table}

\begin{table}[ht]
\caption{Detailed Model Training Parameter Settings of Our Mask Generator (where $b$, $\alpha$ and $\gamma$ denote batch size, initial learning rate and learning rate decay, respectively)}
\label{training}

\begin{center}
\begin{small}
\begin{sc}
\begin{tabular}{c|cccccc}
\toprule
           &  &  & \multicolumn{2}{c} {5-layer} &  \multicolumn{2}{c} {6-layer} \\
           & $b$ & Milestones        & $\alpha$ & $\gamma$ & $\alpha$ & $\gamma$ \\
\midrule
CIFAR10    & 256 & {[}0, 100, 145{]} & 0.01             & 0.1             & 0.001            & 1               \\
CIFAR100   & 256 & {[}0, 100, 145{]} & 0.01             & 0.1             & 0.001            & 1               \\
SVHN       & 256 & {[}0, 100, 145{]} & 0.01             & 0.1             & 0.001            & 1               \\
GTSRB      & 256 & {[}0, 100, 145{]} & 0.01             & 0.1             & 0.001            & 1               \\
Flowers102 & 256 & {[}0, 100, 145{]} & 0.01             & 0.1             & 0.001            & 1               \\
DTD        & 64  & {[}0, 100, 145{]} & 0.01             & 0.1             & 0.001            & 1               \\
UCF101     & 256 & {[}0, 100, 145{]} & 0.01             & 0.1             & 0.001              & 1             \\
Food101    & 256 & {[}0, 100, 145{]} & 0.01             & 0.1             & 0.001             & 1             \\
SUN397     & 256 & {[}0, 100, 145{]} & 0.01             & 0.1             & 0.001            & 1               \\
EuroSAT    & 256 & {[}0, 100, 145{]} & 0.01             & 0.1             & 0.001            & 1               \\
OxfordPets & 64  & {[}0, 100, 145{]} & 0.01             & 0.1             & 0.001             & 1    \\
\bottomrule
\end{tabular}
\end{sc}
\end{small}
\end{center}
\end{table}

\section{Additional Experimental Results}
\subsection{Applying SMM with Different $f_{\rm out}$}\label{Asec:diff_f}
\begin{table*}[t] 
\caption{Performance Improvement When Applying Our Input Reprogramming on Different Label Mapping Methods (the average results are highlighted in grey)}
\label{improve}

\begin{center}
\begin{small}
    
\begin{sc}

\resizebox{0.95\textwidth}{!}{
\begin{tabular}{c|ccc|ccc|ccc}
\toprule
\multicolumn{1}{c|}{$f_{\rm out}$} & \multicolumn{3}{c|}{Iterative Label   Mapping}   & \multicolumn{3}{c|}{Frequent Label Mapping}      & \multicolumn{3}{c}{Random Label Mapping}        \\
\midrule
                                  & w/o Ours & \multicolumn{1}{l}{w Ours} & Improve  & w/o Ours & \multicolumn{1}{l}{w Ours} & Improve  & w/o Ours & \multicolumn{1}{l}{w Ours} & Improve  \\
\midrule
CIFAR10                           & 68.90\%  & 72.80\%                   & +3.90\%  & 71.79\%  & 72.75\%                   & +0.96\%  & 65.68\%  & 69.71\%                   & +4.03\%  \\
CIFAR100                          & 33.80\%  & 39.40\%                   & +5.60\%  & 29.79\%  & 32.35\%                   & +2.56\%  & 16.99\%  & 23.47\%                   & +6.48\%  \\
SVHN                              & 78.30\%  & 84.40\%                   & +6.10\%  & 78.78\%  & 83.73\%                   & +4.95\%  & 77.44\%  & 85.37\%                   & +7.92\%  \\
GTSRB                             & 76.80\%  & 80.40\%                   & +3.60\%  & 74.76\%  & 80.90\%                   & +6.14\%  & 69.60\%  & 82.38\%                   & +12.79\% \\
Flowers102                        & 23.20\%  & 38.70\%                   & +15.50\% & 17.78\%  & 32.16\%                   & +14.37\% & 12.34\%  & 37.68\%                   & +25.33\% \\
DTD                               & 29.00\%  & 33.60\%                   & +4.60\%  & 30.14\%  & 34.28\%                   & +4.14\%  & 14.60\%  & 19.74\%                   & +5.14\%  \\
UCF101                            & 24.40\%  & 28.70\%                   & +4.30\%  & 22.71\%  & 25.72\%                   & +3.01\%  & 9.04\%   & 16.71\%                   & +7.67\%  \\
Food101                           & 13.20\%  & 17.50\%                   & +4.30\%  & 11.58\%  & 15.21\%                   & +3.62\%  & 7.15\%   & 15.86\%                   & +8.71\%  \\
SUN397                            & 13.40\%  & 16.00\%                   & +2.60\%  & 13.45\%  & 15.45\%                   & +1.99\%  & 1.05\%   & 3.35\%                    & +2.29\%  \\
EuroSAT                           & 84.30\%  & 92.20\%                   & +7.90\%  & 86.00\%  & 92.67\%                   & +6.67\%  & 84.49\%  & 94.47\%                   & +9.98\%  \\
OxfordPets                        & 70.00\%  & 74.10\%                   & +4.10\%  & 69.66\%  & 72.83\%                   & +3.16\%  & 8.89\%   & 16.84\%                   & +7.96\% \\ \rowcolor{gray!30} Average & 46.85\%&52.53\%&+5.68\%&46.04\%&50.73\%&+4.69\%&33.39\%&42.32\%&+8.94\% \\
\bottomrule
\end{tabular}}
\end{sc}
\end{small}
\end{center}
\end{table*}

As mentioned before, and as shown in Appendix \ref{prob}, input VR is agnostic of the output label mapping method. Thus, our SMM can be applied to different output label methods other than Ilm. Experimental results are presented in Table \ref{improve}.

Our method improves the performance of all output mapping methods. In most cases, the worse the output mapping method is, the more pronounced the improvement of SMM will be. When there is sufficient training data (e.g., GTSRB, SVHN, CIFAR10 and Food101), adding SMM can compensate for the worse-performing label mapping methods. With SMM, these methods also produce competitive results.

\subsection{Analysis of Learning Curves} \label{trainingcurves}
\begin{figure}[!ht]
    \centering
    \includegraphics[width=\linewidth]{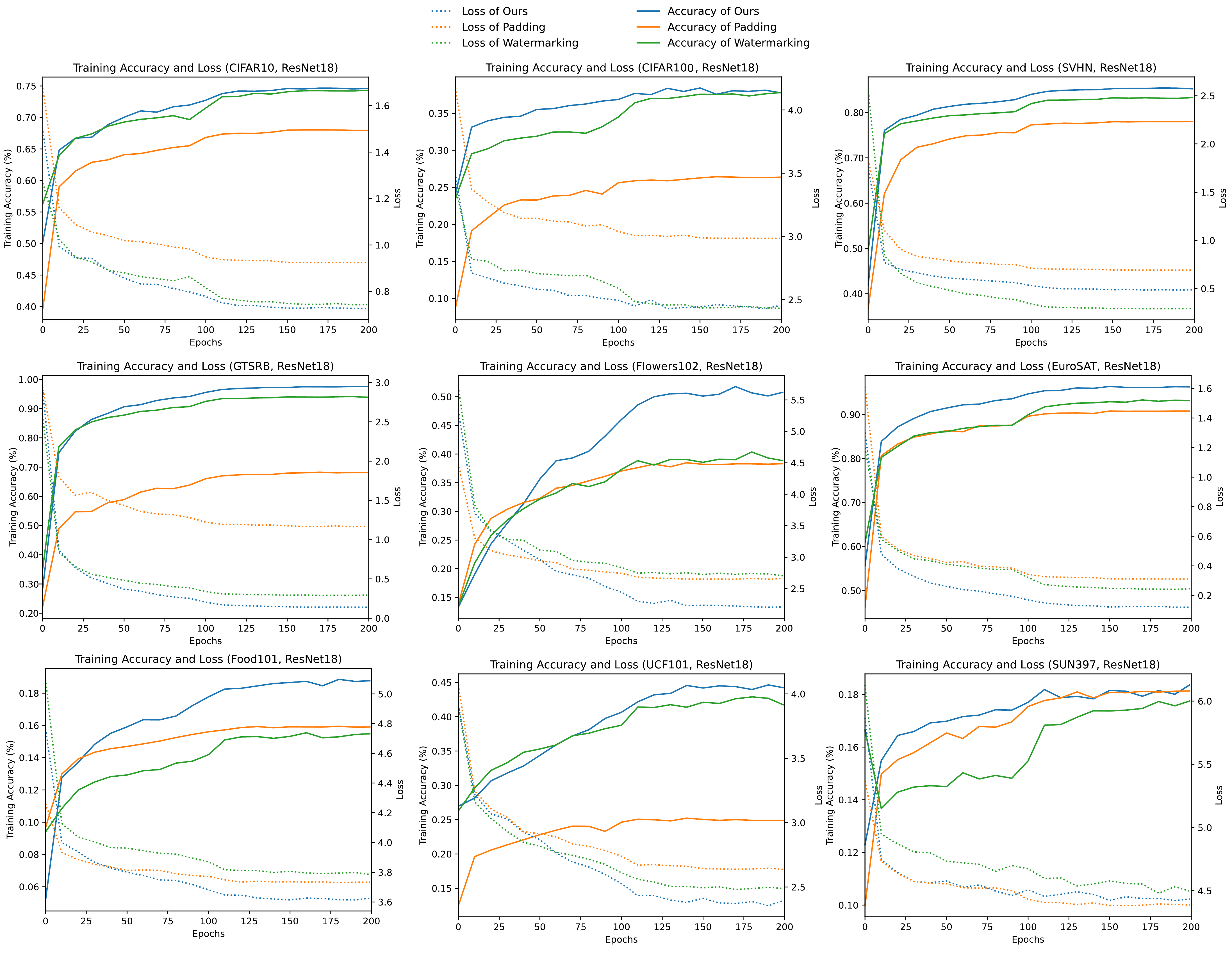}
    \caption{Training Accuracy and Loss of Different Reprogramming Methods}
    \label{fig:trainloss}
\end{figure}

\begin{figure}[!ht]
    \centering
    \includegraphics[width=\linewidth]{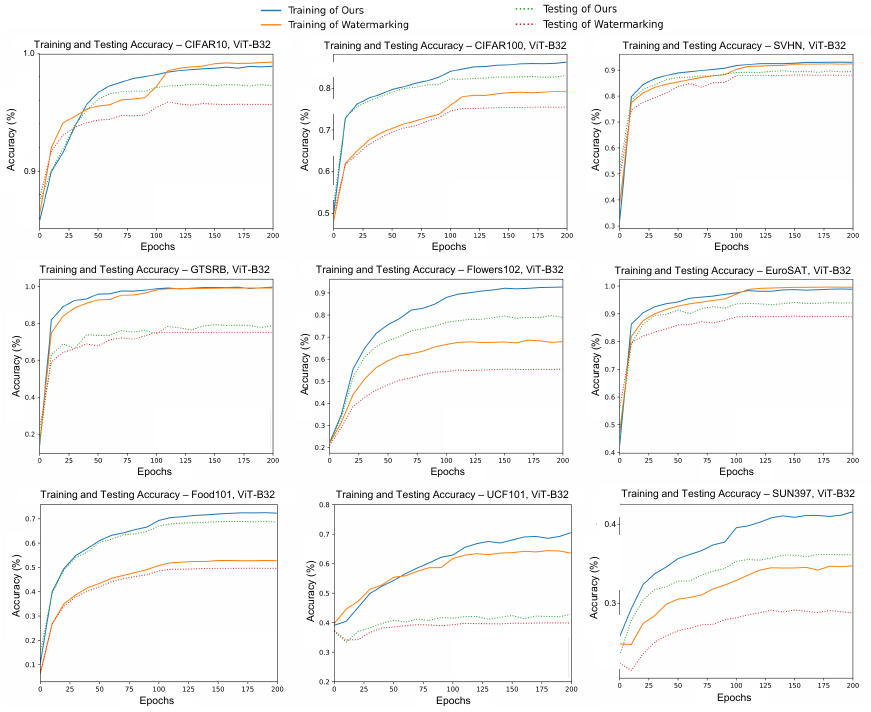}
    \caption{Training Accuracy and Testing Accuracy with and without Our Method}
    \label{fig:traintestacc}
\end{figure}

Figure \ref{fig:trainloss} shows the training accuracy and loss throughout learning iterations using ResNet-18 as the pre-trained backbone. We see that our SMM yields a higher training accuracy and lower loss for most cases. 

When using a more sophisticated pre-trained network, e.g., ViT, as is shown in Figure \ref{fig:traintestacc}, the training accuracy without SMM may meet with or even exceed that of using SMM. However, this appears to be a case of over-fitting, where training accuracy is approaching 1 and test accuracy is still low without using SMM.

In general, for smaller classifiers such as ResNet-18, adding our model helps better reduce training loss and improve accuracy, while for more sophisticated classifiers such as ViT-B32 where the training accuracy is already high, adding our SMM model helps prevent over-fitting and improve the testing accuracy.

\begin{table}[ht]
\caption{Training and Testing Accuracy with Enlarged $f_{\rm mask}$ (using EuroSAT, ResNet-18)}
\label{tab:error}
\centering
\begin{small}
\begin{sc}
\begin{tabular}{c|cccccc}
\toprule
$f_{\rm mask}$         & Small & Medium (ours) & Large  & x-Large & xx-Large & xxx-Large \\
\midrule
Parameters             & 7203  & 26499         & 101379 & 396291  & 1566723  & 6230019   \\
Training Accuracy (\%) & 94.9  & 96.2          & 96.4   & 97.3    & 97.7     & 98.1      \\
Testing Accuracy (\%)  & 91.7  & 92.2          & 92.2   & 93.1    & 93.5     & 93.2     \\
\bottomrule
\end{tabular}
\end{sc}
\end{small}
\end{table}

\subsection{More Discussion about the Estimation Error}
\label{app:error}
A higher estimation error generally implies an increased risk of model over-fitting to the training data. This observation can be corroborated by comparing the disparities in training and testing performance. For instance, as depicted in Figure~\ref{fig:traintestacc}, employing a more sophisticated pre-trained network such as ViT with a mask generator $f_{\rm mask}$ shown in Figure~\ref{fig:6layerarchitecture} across some tasks like CIFAR10, SVHN, and GTSRB, the training accuracy tends towards 100\% for both shared patterns $\mathcal{F}^{\rm shr}(f'_{\rm P})$ (i.e., `Watermarking' in Figure~\ref{fig:traintestacc}) and SMM patterns $\mathcal{F}^{\rm smm}(f'_{\rm P})$ (i.e., `Ours' in Figure~\ref{fig:traintestacc}). Despite this, $\mathcal{F}^{\rm smm}(f'_{\rm P})$ maintains a test accuracy that is not inferior to that of shared patterns. It suggests that our method SMM does not suffer from more significant over-fitting than shared masking, resulting in negligible potential estimation error.

However, when $f_{\rm mask}$ is enlarged with increased number of parameters, the additional estimation error of $\mathcal{F}^{\rm smm}(f'_{\rm P})$ may no longer be negligible and will impact the excess risk. The relationship between the number of parameters in $f_{\rm mask}$ and the estimation error is influenced by various factors, including the specific target tasks, the volume of training data, the size of well-trained models, and the design of our generation model, etc. Through experiments, we will be able to estimate when the number of parameters begins to impact estimation error, potentially leading to over-fitting. For instance, in Table~\ref{tab:error}, we employ our generation model $f_{\rm mask}$ on the EuroSAT dataset, with ResNet-18 being the well-trained model. By progressively doubling the number of intermediate channels while maintaining the architecture of $f_{\rm mask}$, we investigate how the model size affects performance.

Through the results of Table~\ref{tab:error}, we come to the following conclusions:
(1) As the number of parameters continues to increase, although the training accuracy slowly increases, the test accuracy may even decrease, implying that the estimation error becomes more and more noticeable.
(2) Under this situation (i.e., EuroSAT, ResNet-18), when the size of $f_{\rm mask}$ is close to the same order of magnitude as the well-trained model, the estimation error should not be overlooked.
(3) A larger model with the best test accuracy may not be optimal because of too many parameters. Our work strikes a balance between the number of parameters and test accuracy.

\subsection{Further Analysis of the Performance of SMM}
\textbf{More Discussion of SMM Abnormal Cases.}
\label{app:abnormal}
In Section \ref{exp}, we have briefly analyzed abnormal performance in Table~\ref{resnetresults} and Table~\ref{ViTresults}. In this section, we will provide a more comprehensive discussion. Here, we outline detailed discussions regarding abnormal performance:
\begin{itemize}
    \item ResNet-18, DTD: As shown in Figure~\ref{fig:dtd}, the DTD dataset contains a significant amount of texture features. Therefore, for relatively simple well-trained models, introducing reprogramming noise in the form of watermarking may affect the original features of the images. It can be observed that when the watermarking area is small (Narrow), the effect is better compared to when it is large (Full), and our method is also affected by this factor. However, the padding-based method preserves the original pixels of the image and only introduces reprogramming noise around them, thereby achieving relatively good results.
    \item ViT-B32, EuroSAT: This is because EuroSAT is one of the target tasks with the least task complexity. When using a large-scale network like ViT, the resizing-based method leads to over-fitting. As evident in the third column of the second row in Figure~\ref{fig:traintestacc}, the training accuracy is already close to 1. Therefore, in this scenario, the padding-based method yields slightly better test results compared to our method (which also belongs to resizing-based methods).
\end{itemize}

\begin{table}[ht]
\caption{An Ineffective   Case of Input Reprogramming  -   StanfordCars (Mean \% ± Std \%)}
\label{stanfordcars}

\begin{center}
\begin{small}
\begin{sc}
\begin{tabular}{cccccc}
\toprule
Method        & Pad         & Narrow     & Medium     & Full       & Ours       \\
\midrule

ResNet-18     & 4.5±0.1     & 3.6±0.1    & 3.6±0.1    & 3.4±0.1    & 2.9±0.2    \\
ResNet-50     & 4.7±0.2     & 4.7±0.1    & 4.7±0.2    & 4.6±0.1    & 3.0±0.6    \\
ViT-B32       & 4.7±0.6     & 7.7±0.2    & 8.3±0.3    & 5.0±0.0    & 4.8±0.9   \\
\bottomrule
\end{tabular}
\end{sc}
\end{small}
\end{center}
\end{table}
\textbf{SMM on An Ineffective Case of Input Reprogramming.} \label{stanfordappend}
All input visual reprogramming methods seem ineffective on fine-grained recognition tasks where subtle appearance differences should be detected. As shown in Table \ref{stanfordcars}, in the classification of StanfordCars, where 196 types of cars are to be classified, the accuracy of all input VR methods is below 10 \%, indicating the failure of VR methods in this fine-grained recognition tasks. Adding our SMM module will not improve performance when VR methods fail.

\section{Additional Discussion about Input VR Compared with Finetuning}
\label{app:compare}
\subsection{Advantages of VR in Dealing with Distorted Input Images}

\begin{table}[h]
\caption{Performance of Finetuning (LoRA) and SMM Facing Target Tasks with Different Input Image Sizes (Accyracy \%, using ViT-L with a 384$\times$384 input as the well-trained model, average results are calculated on all four tasks with 32$\times$32 inputs and all seven tasks with 128$\times$128 inputs)}
\centering
\begin{sc}
\begin{small}
\begin{tabular}{cc|cccc|cc}
\toprule
                & \makecell{Extra \\ Parameters} & CIFAR10       & CIFAR100      & SVHN          & GTSRB         & \makecell{Average \\ (32$\times$32)} & \makecell{Average \\ (128$\times$128)} \\
\midrule
Finetuning-LoRA & 0.60M            & 95.9          & 83.6          & 65.3          & 66.6          & 77.9            & 83.4              \\
Our SMM         & 0.54M            & \textbf{97.4} & \textbf{87.3} & \textbf{91.0} & \textbf{84.2} & \textbf{90.0}   & \textbf{83.5}  \\
\bottomrule
\end{tabular}
\label{tab:lora}
\centering
\end{small}
\end{sc}
\end{table}
In this section, we will compare the results of our SMM with finetuning-based methods to show the advantages of input VR in dealing with distorted input images. \textit{Low-rank adaptation} (LoRA) \cite{hu2021lora} is an efficient finetuning-based transfer method proposed based on large language models for natural language processing, which has been adapted to ViT \cite{zhu2023melo}. Here, we compare SMM for ViT with LoRA for ViT, which are representative methods that belong to input VR and finetuning, respectively.

Since LoRA for ViT already includes finetuning the fully connected layers, we also incorporate it in SMM. All training settings are kept the same. We set the rank of LoRA to be six, resulting in an additional parameter number being 0.60M (without counting the fully connected layers), which will be comparable to that of input VR and SMM (being 0.54M) for fairness. ViT-Large with the input size being 384$\times$384 is applied, and the learning rate is 0.01, running 10 epochs in total. Therefore, for both methods, the target training samples will be resized before input. VR mainly trains parameters in the input space before well-trained models, whereas LoRA injects parameters into layers of ViT. Results are listed in Table~\ref{tab:lora}.

The results of target tasks with the input size being 128$\times$128 are similar. However, it is observed that for those target tasks with lower resolution (e.g., CIFAR10/100, SVHN, GTSRB), our SMM appears to perform better. This is likely because when a 32$\times$32 image is resized to 384$\times$384, it may become distorted, thus affecting the performance of target tasks. This distortion is especially noticeable on tasks with simple features, such as SVHN and GTSRB. Since VR modifies the input space, it effectively addresses this issue of significant differences in the input image sizes of pre-trained and target tasks.

\subsection{Advantages of VR in Being Orthogonal to Finetuning-based Methods}

\begin{table}[h]
\caption{Performance of Finetuning the Fully-Connected Layers (Finetuning-FC) without or with our SMM Module (Accuracy \%, using ResNet-50 as the well-trained model)}
\centering
\begin{sc}
\begin{small}
\begin{tabular}{c|cccccc}
\toprule
                        & CIFAR10 & CIFAR100 & SVHN   & GTSRB   & Flowers102 & DTD     \\
\midrule
Finetuning-fc           & 90.1    & 70.7     & 63.5   & 77.8    & \textbf{90.9}       & 67.6    \\
Finetuning-fc + Our SMM & \textbf{91.2}    & \textbf{72.4}     & \textbf{86.9}   & \textbf{85.2 }   & \textbf{90.9}       & \textbf{68.2}    \\
\bottomrule
\toprule
                        & UCF101  & Food101  & SUN397 & EuroSAT & OxfordPets & Average \\
\midrule
Finetuning-fc           & 70.8    & 57.6     & 53.5   & 95.7    & 90.4       & 75.3   \\
Finetuning-fc + Our SMM & \textbf{72.0}      & \textbf{59.6}     & \textbf{57.9}   & \textbf{95.8}    & \textbf{90.6}       & \textbf{79.2}  \\
\bottomrule
\end{tabular}  
\label{tab:ft}
\end{small}
\end{sc}
\end{table}

Since finetuning and reprogramming are orthogonal because finetuning modifies the model while reprogramming modifies the input and output spaces. 
Input VR can also be combined with finetuning-based methods.
In this section, we will add the input VR module (i.e, using SMM as an example) to finetuning-based methods and analyze the performance gain.
A widely-used method - finetuning the fully connected layer (named `Finetuning-FC') - is employed as the baseline method. 
Using ResNet-50 as the well-trained model, we add our SMM input VR module to `Finetuning-FC' to demonstrate the effectiveness of our module.

Results are shown in Tabel~\ref{tab:ft}. Utilizing our module achieves an average accuracy of about 4\% higher than solely finetuning the fully connected layers. Conclusively, input VR can be attached to finetuning-based methods to improve performance.

\subsection{Strengths and Weaknesses of Input Reprogramming in Visual Tasks}
This part includes a conclusion of the strengths and weaknesses of Input VR, compared with finetuning-based methods.

\subsubsection{Strengths}
\begin{itemize}
    \item The \textit{parameter numbers} of VR tend to be \textit{negligible} considering the size of well-trained models. Besides, the parameter numbers in VR are solely determined by the size of a single input image, independent of well-trained models, and remain fixed as the well-trained model size grows.
    \item VR is suitable for all well-trained models, \textit{regardless of the architecture}, whereas finetuning-based methods are usually designed for a specific architecture (e.g., LoRA is specifically designed for ViT).
    \item VR improves the performance of the target task by altering the input and output space, and analyzing these changes may help \textit{understand why} the model can also perform well in the target domain.
    \item  By changing the input and output spaces while fixing the well-trained model, VR \textit{avoids practical issues} such as catastrophic forgetting (i.e., the well-trained model may lose previously learned representations when being finetuned for new tasks).
    \item VR \textit{can be attached to} most mainstream finetuning methods to further improve performance.
    \item In future research, VR could also utilize the well-trained model as a \textit{black box}. This approach might prove useful for re-purposing models that only offer an application programming interface.
\end{itemize}

\subsubsection{Weaknesses}
\begin{itemize}
    \item When target tasks are more \textit{challenging} than the tasks well-trained models have been trained on, merely adjusting the input space may \textit{not be sufficient} for satisfied performance. This poses a challenge for VR.
    \item For better performance approaching re-training or fully finetuning, integrating VR with other finetuning methods appears necessary (e.g., VR may be combined with finetuning the fully connected layer). \textit{How to train the combined model more effectively} remains a task for future research.
\end{itemize}

\section{Additional Visualization Results} \label{visualresults}
\begin{figure}[t]
    \centering
    \includegraphics[width=\linewidth]{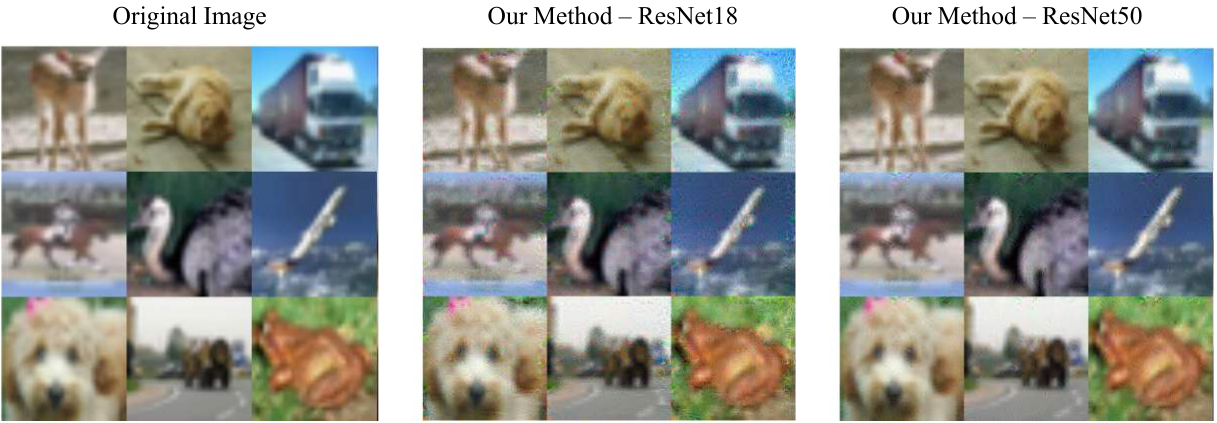}
    \caption{Original Images and Visual Reprogramming Results on CIFAR10}
    \label{fig:cifar10}
\end{figure}

\begin{figure}[t]
    \centering
    \includegraphics[width=\linewidth]{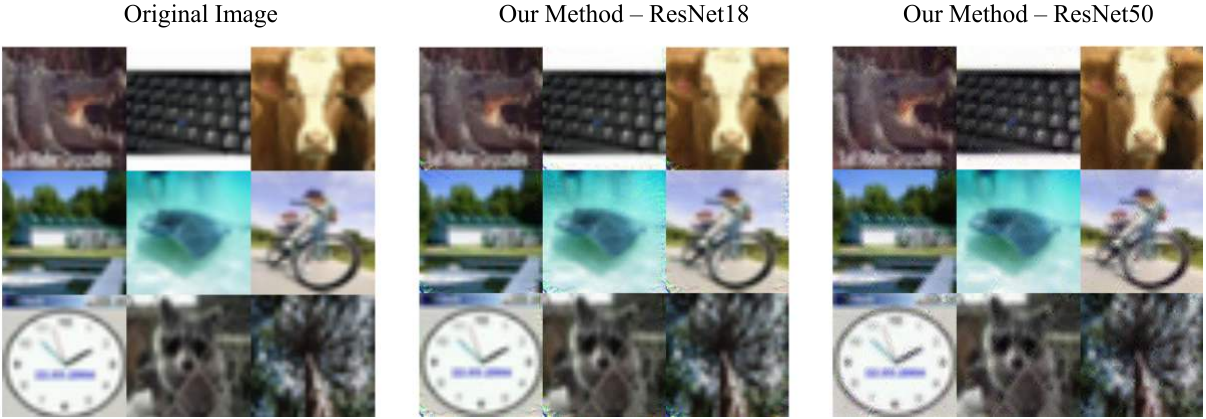}
    \caption{Original Images and Visual Reprogramming Results on CIFAR100}
    \label{fig:cifar100}
\end{figure}

\begin{figure}[t]
    \centering
    \includegraphics[width=\linewidth]{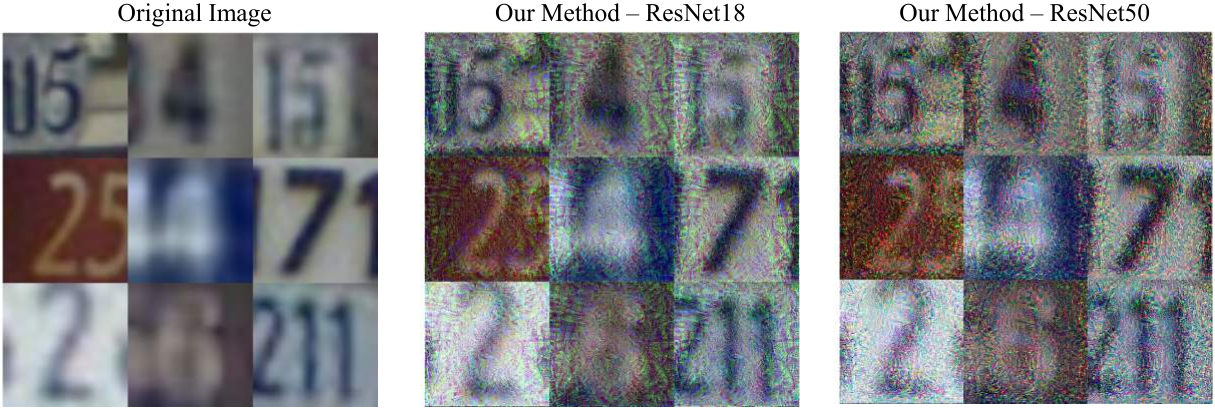}
    \caption{Original Images and Visual Reprogramming Results on SVHN}
    \label{fig:svhn}
\end{figure}

\begin{figure}[t]
    \centering
    \includegraphics[width=\linewidth]{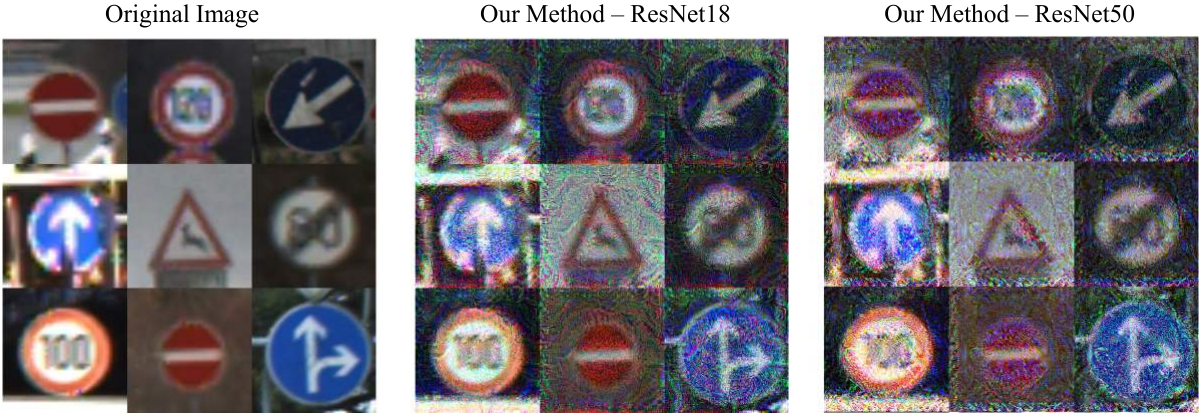}
    \caption{Original Images and Visual Reprogramming Results on GTSRB}
    \label{fig:gtsrb}
\end{figure}

\begin{figure}[t]
    \centering
    \includegraphics[width=\linewidth]{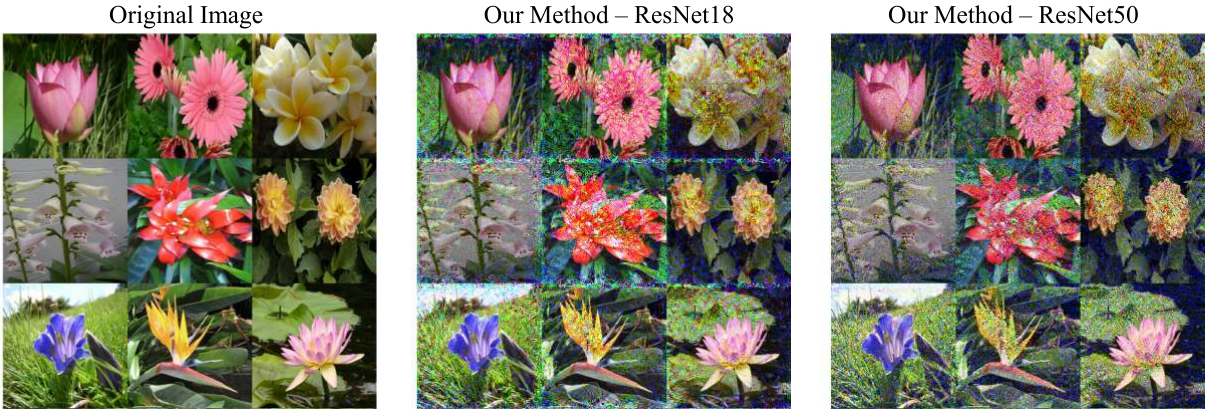}
    \caption{Original Images and Visual Reprogramming Results on Flowers102}
    \label{fig:flowers102}
\end{figure}

\begin{figure}[t]
    \centering
    \includegraphics[width=\linewidth]{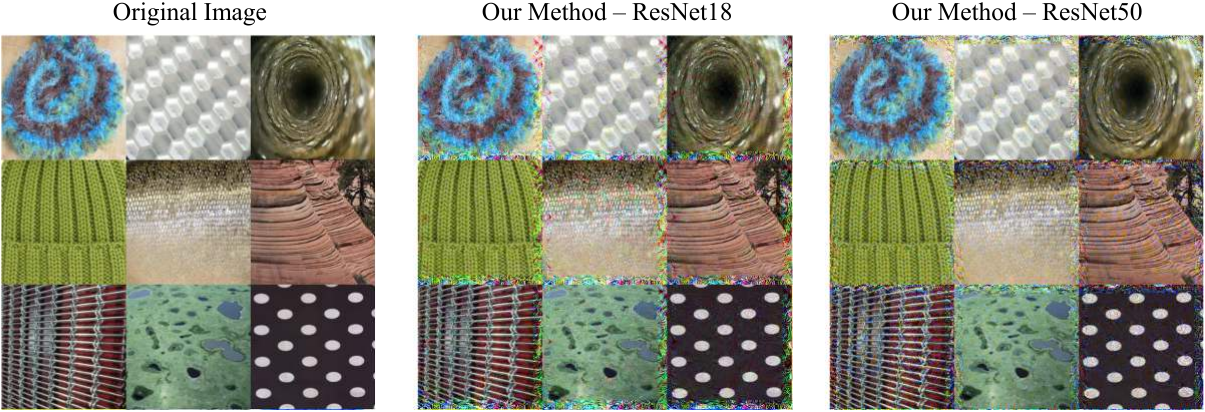}
    \caption{Original Images and Visual Reprogramming Results on DTD}
    \label{fig:dtd}
\end{figure}

\begin{figure}[t]
    \centering
    \includegraphics[width=\linewidth]{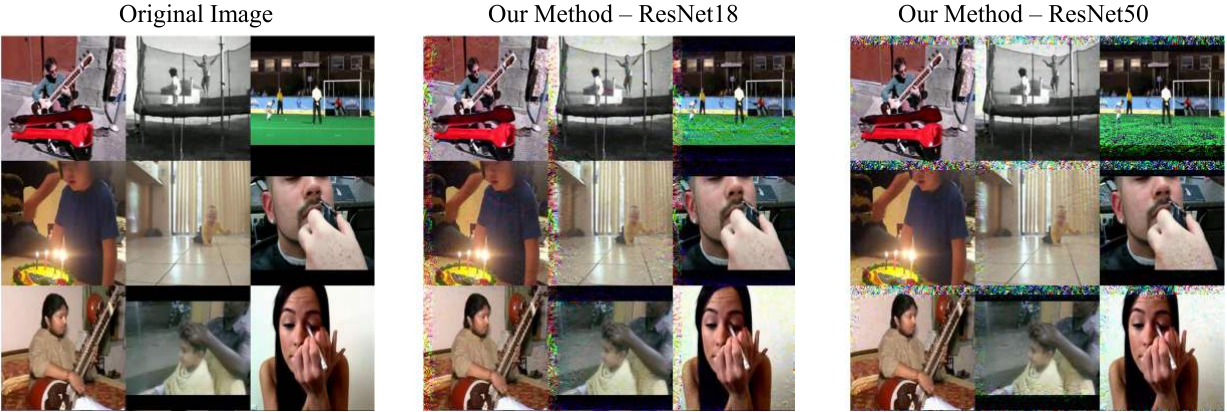}
    \caption{Original Images and Visual Reprogramming Results on UCF101}
    \label{fig:ucf101}
\end{figure}

\begin{figure}[t]
    \centering
    \includegraphics[width=\linewidth]{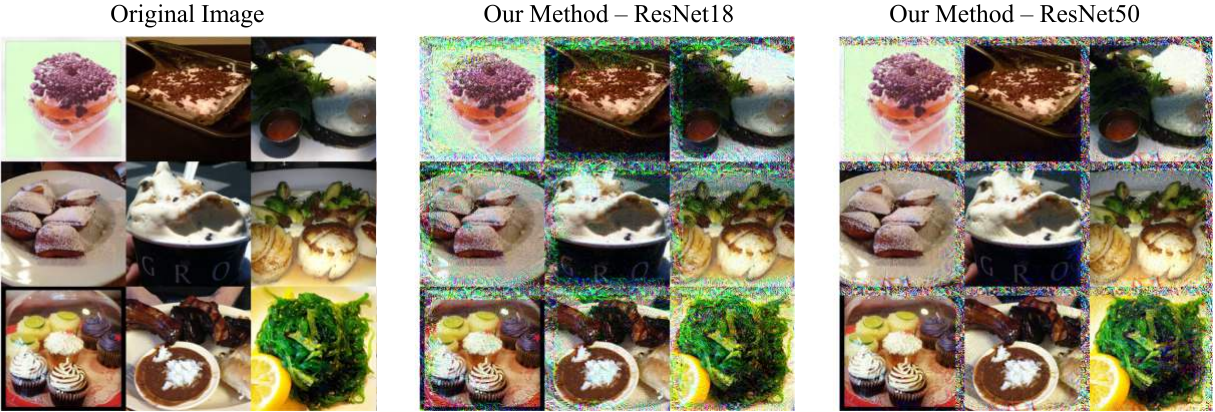}
    \caption{Original Images and Visual Reprogramming Results on Food101}
    \label{fig:food101}
\end{figure}

\begin{figure}[t]
    \centering
    \includegraphics[width=\linewidth]{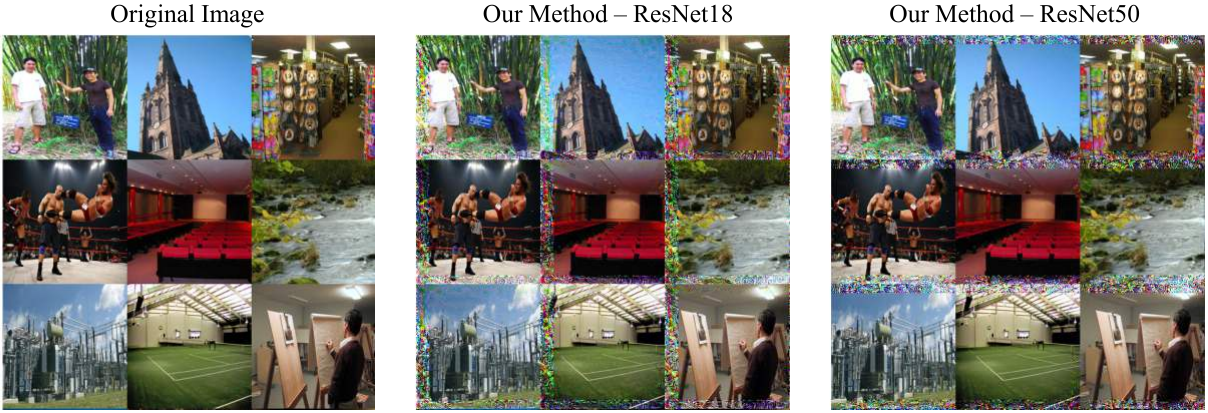}
    \caption{Original Images and Visual Reprogramming Results on SUN397}
    \label{fig:sun397}
\end{figure}

\begin{figure}[t]
    \centering
    \includegraphics[width=\linewidth]{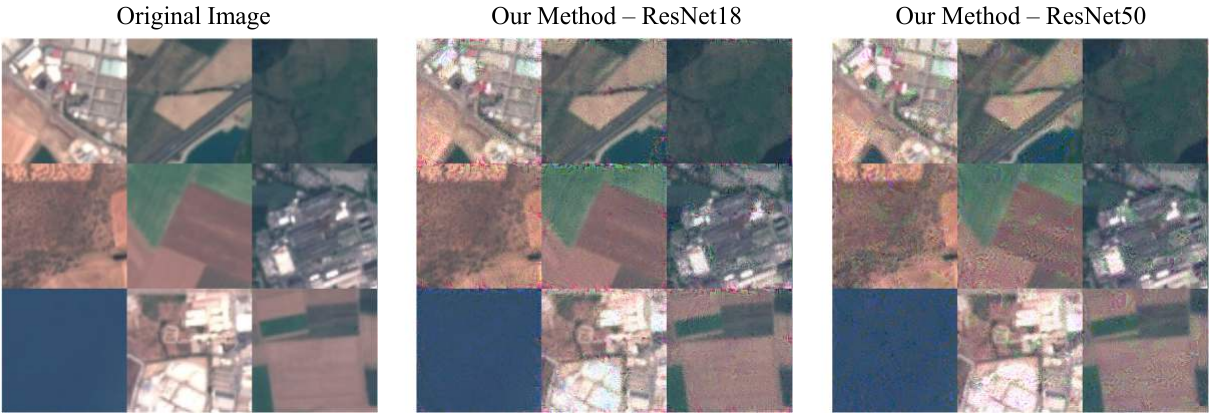}
    \caption{Original Images and Visual Reprogramming Results on EuroSAT}
    \label{fig:eurosat}
\end{figure}

\begin{figure}[t]
    \centering
    \includegraphics[width=\linewidth]{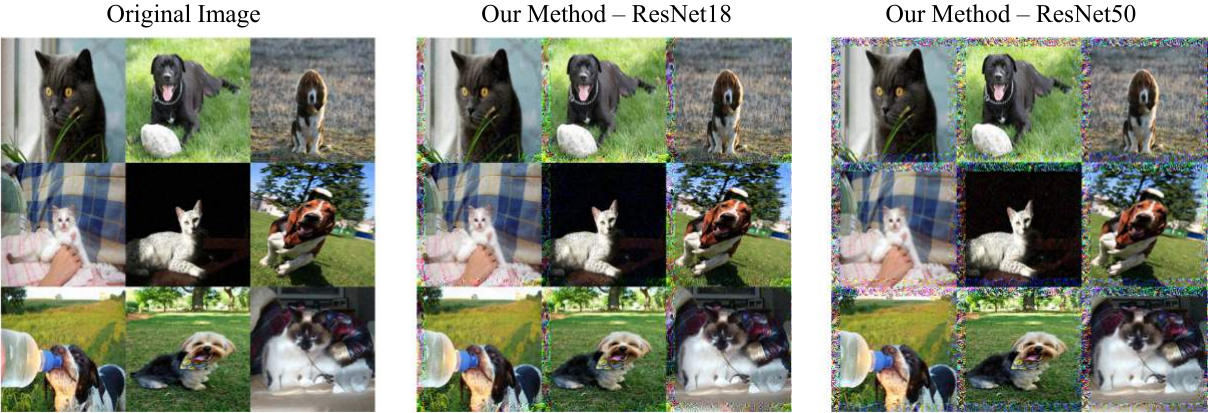}
    \caption{Original Images and Visual Reprogramming Results on OxfordPets}
    \label{fig:oxfordpets}
\end{figure}

Figure \ref{fig:cifar10}-\ref{fig:oxfordpets} show sample images of the VR results of SMM on 11 datasets. These figures show that (1) our VR method does not alter the input space heavily; it only adds noise within a limited range, which ensures that the original images remain intact; (2) the more different the target domain is (e.g., GTSRB and SVHN), the more pronounced the noise pattern will be; (3) on datasets that prefer VR to be a narrow padding-sized watermark, SMM will convergence to a similar situation, that is, the noise at the outer frame of the images is much greater than that inside the images (e.g., UCF101, Food101, OxfordPets and SUN397).

\end{document}